\documentclass[11pt]{article}
\usepackage{amsthm, amsmath,amssymb}
\usepackage[margin=1in]{geometry}
\usepackage[boxruled]{algorithm2e}
\usepackage{algpseudocode}
\usepackage{comment}
\usepackage{bbm}
\usepackage{float}
\usepackage{graphicx}
\newtheorem{claim}{Claim}
\newtheorem{cor}{Corollary}

\newtheorem{theorem}{Theorem}
\newtheorem{definition}{Definition}
\newtheorem{lemma}{Lemma}
\newtheorem{remark}{Remark}
\newcommand{\D}{{\cal D}}
\newcommand{\sample}{{\cal S}}
\newcommand{\E}{{\mathrm{E}}}
\newcommand{\R}{\mathbb{R}}

\usepackage{enumitem}

\setenumerate{noitemsep,topsep=0pt,parsep=1pt,partopsep=1pt}
\setitemize{noitemsep,topsep=0pt,parsep=1pt,partopsep=1pt}
 \author{Maria-Florina Balcan \and
Ellen Vitercik \and
Colin White}

\date{}

\title{Learning Combinatorial Functions from Pairwise Comparisons
\footnote{Authors' addresses: \texttt{\{ninamf,vitercik,crwhite\}@cs.cmu.edu}.}
}

\begin{document}

\maketitle     

\begin{abstract}
A large body of work in machine learning has focused
on the problem of learning a close approximation to an underlying combinatorial function,
given a small set of labeled examples.
However, for real-valued functions, cardinal labels might not be accessible, or it may be difficult for an expert to 
consistently assign real-valued labels over the entire set of examples. For instance, it is 
notoriously hard for consumers to reliably assign values to bundles of merchandise. 
Instead, it might be much easier for a consumer to report which of two bundles she likes better.

With this motivation in mind, we consider an alternative learning model, wherein the algorithm must learn the underlying function up to pairwise comparisons, from pairwise comparisons. In this model, we present a series of novel algorithms that learn over a wide variety of combinatorial function classes. These range from graph functions to broad classes of valuation functions that are fundamentally important in microeconomic theory, the analysis of social networks, and machine learning, such as coverage, submodular, XOS, and subadditive functions, as well as functions with sparse Fourier support.
\end{abstract}

\section{Introduction}
The problem of ranking based on pairwise comparisons is present in many application domains ranging from algorithmic game theory~\cite{agt-book} to computational finance~\cite{comp-finance-book} to social networks~\cite{jk-book}. For example, a business might wish to learn its consumers' combinatorial valuation functions, since this will allow them to better set prices, choose which goods to sell as bundles, and determine inventory levels. Previous work on learning valuation functions has concentrated on the model in which the learning algorithm is given access to a set of examples (bundles of goods) which are labeled by the underlying valuation function \cite{Balcan11submodular,Sketching,Feldman2012optimal,balcan2012learning,iyer2013curvature}. However, for real-valued functions, this cardinal data may not be accessible. Indeed, it may be difficult for a consumer to provide the real-valued number corresponding
to her valuation for a bundle of goods. Instead, it might be more natural for her to express
whether she likes one bundle of goods more than another. After all, it is well-known that humans are significantly better at comparing than scoring
\cite{barnett2003modern,stewart2005absolute}. Therefore, we may hope to learn a consumer's valuation function up to pairwise comparisons, from pairwise comparisons.

Of course, economics is not the only field where it would be valuable to learn an underlying function up to pairwise comparisons. Research on judgment elicitation through pairwise comparisons is a fundamental problem in fields outside of computer science, ranging from psychology to economics to statistics, as well as many others \cite{heldsinger2010using,chen2013pairwise,
bradley1952rank,barnett2003modern}. For example, in a social network, one might wish to learn the influence of subgroups and individuals, but it could be difficult to consistently assign real-valued numbers as measurements of this influence. Rather, it might be easier to simply answer which of two subgroups is more influential. Although the number of subgroups in a social network may be exponential in the number of nodes, through a polynomial number of such queries, we may hope to learn a pairwise comparison function that allows us to accurately predict which of any two subgroups is more influential.

\subsection{Our Results}
In this paper, we prove that many classes of combinatorial functions can be learned up to comparisons.
Our formal definition of what it means to learn a function up to comparisons is similar to the PAC setting: we say that a class of functions is \emph{comparator-learnable} if there exists is an efficient algorithm which outputs a comparison function
that, with high probability
over the choice of examples, has small error over the distribution. For some function classes, we require that the function value of two sets be sufficiently far apart in order to guarantee that the learned comparator predicts accurately on those sets.

More formally, in Section~\ref{sec:multiplicativeError}, we show that for a fixed class $\mathcal{F}$ of combinatorial
functions which map $2^{[n]}$ to $\R$, if any function in $\mathcal{F}$ can be multiplicatively approximated up to a factor of $\alpha(n)$
by some power of a linear function,
then we can learn any function in $\mathcal{F}$ up to comparisons on pairs of sets whose values differ by at least an $\alpha(n)$ multiplicative factor. In this case, we say that $\mathcal{F}$ is \emph{comparator-learnable with separation $\alpha(n)$.} Our results are summarized in Tables \ref{tab:results} and \ref{tab:results2}.
Using existing approximation results
\cite{goemans2009approximating,balcan2012learning}, we immediately conclude
that several broad classes of combinatorial functions are comparator-learnable, including many that are ubiquitous in microeconomic theory. These include the nested classes of monotone submodular, XOS, and subadditive functions, all of which are used to model consumer preferences that exhibit diminishing marginal utility. In particular, we show that submodular functions are comparator-learnable with separation $\alpha(n) = \sqrt{n}$ and provide a nearly-matching lower bound of $\alpha(n) = \tilde{\Omega}(n^{1/3})$. Further, we show that the classes of XOS and subadditive functions are comparator-learnable with separation $\alpha(n) = \tilde{\Theta}(\sqrt{n})$.

We also rely on results from \cite{iyer2013curvature} and \cite{balcan2012learning} to achieve stronger bounds for submodular functions if the \emph{curvature} is small. Curvature is a well-studied measure of submodular complexity which quantifies how close a function is to being fully additive. We prove that the separation factor approaches 1 (which is optimal) as the function class approaches full additivity, i.e. as the maximum curvature approaches 0. Further, for XOS functions with polynomially-many SUM trees, we show that the separation factor decreases as a function of the number of trees. In this way, the more structured a class in question is, the stronger our results in Section~\ref{sec:multiplicativeError} are.

\renewcommand{\arraystretch}{1.5}
\begin{table}[h]
\begin{center}
\begin{tabular}{|p{6cm}||p{3.5cm}|p{4cm}|}
\hline
\textbf{Function Class}                                                                    & \textbf{Separation}   & \textbf{Sample Complexity} \\ \hline\hline
Subadditive                                                                      & $\tilde{\Theta}(\sqrt{n})$ & $\tilde{O}\left(n/\epsilon^3\right)$                          \\\hline
XOS                                                                              & $\tilde{\Theta}(\sqrt{n})$                     & $\tilde{O}\left(n/\epsilon^3\right)$ \\\hline
Submodular                                                                                 &   $O\left(\sqrt{n}\right), \tilde{\Omega}\left(\sqrt[3]{n}\right)$                    &                           $\tilde{O}\left(n/\epsilon^3\right)$ \\\hline
Submodular with curvature $\kappa$                                                         & $O\left(\min \left\{\sqrt{n}, \frac{1}{1-\kappa}\right\}\right),$\newline $\frac{n^{1/3}}{O(\kappa\log n)+(1-\kappa)n^{1/3}}$       & $\tilde{O}\left(n/\epsilon^3\right)$            \\\hline
XOS with $R$ SUM trees                                                                     & $O\left(R^{\xi}\right) $, where $\xi > 0$                     &            $\tilde{O}\left(n^{1/\xi}/\epsilon^3\right)$                \\\hline
$|\mathcal{P}|$-sparse Fourier support functions &                      1 &   $\tilde{O}\left(|\mathcal{P}|/\epsilon^2\right)$                         \\\hline
Valuation functions with \newline $k$-limited nonlinear interactions&                      1 &           $\tilde{O}\left(n^k/\epsilon^2\right)$                 \\\hline
Coverage functions                                                                         & $1+\epsilon$                      &    $\tilde{O}\left(n^3/\epsilon^5\right)$   \\\hline
\end{tabular}
\end{center}
\caption{Using our general algorithmic framework, we prove that these combinatorial function classes are comparator learnable with the associated multiplicative separation factors.}\label{tab:results}
\end{table}

\renewcommand{\arraystretch}{1.5}
\begin{table}[h]
\begin{center}
\begin{tabular}{|p{6cm}||p{3.5cm}|p{4cm}|}
\hline
\textbf{Function Class}                                                                    & \textbf{Separation}   & \textbf{Sample Complexity} \\ \hline\hline
XOS functions with \newline distributional assumptions and range in {[}0,1{]}                 &    $\beta \in (0,1)$                  &   $\tilde{O}\left(n^{O\left(\frac{1}{\gamma}\right)}/\epsilon^3\right)$, \newline where $\gamma = \tilde{O}\left(\beta/\epsilon^{3/2}\right)$                         \\\hline
Submodular functions with \newline distributional assumptions and range in {[}0,1{]}                 &     $\beta \in (0,1)$                  &   $\tilde{O}\left(n^{O\left(\frac{1}{\gamma^{4/5}}\log \frac{1}{\gamma}\right)}/\epsilon^3\right)$, \newline where $\gamma = \tilde{O}\left(\beta/\epsilon^{3/2}\right)$                         \\\hline

\end{tabular}
\end{center}
\caption{We also show that our algorithmic framework can be extended to learn over classes of combinatorial function classes on pairs of sets whose values differ by an additive factor $\beta$, for any $\beta \in (0,1)$. For both XOS and submodular functions, we assume that the underlying distribution over subsets of $[n]$ is uniform in order to derive the additive guarantees.}\label{tab:results2}
\end{table}

In Section~\ref{sec:multiplicativeError} we only guarantee the accuracy of the learned comparator on pairs of sets whose values differ by a sufficiently large multiplicative factor. We show in Section \ref{sec:additiveError} that if the underlying distribution over subsets of $[n]$ is uniform, then we can take advantage of key insights regarding the Fourier spectrum of monotone submodular functions with range in $[0,1]$, presented in \cite{FeldmanV15}, to learn such a function up to comparisons on pairs of sets whose values differ by a sufficiently large \emph{additive} factor. We extend this result to XOS functions with range in $[0,1]$ as well.

In Section \ref{sec:otherCombinatorial},
we show that our algorithm from Section \ref{sec:multiplicativeError} applies to a wide range of other classes
of combinatorial functions. In particular, we
present results for functions with sparse Fourier support \cite{stobbe2012learning} and
functions with bounded nonlinear interactions
\cite{vainsencher2011bundle}.
For these more structured function classes we demonstrate a much better $\alpha(n) = 1$, meaning we do not need to assume $f(S)$ and $f(S')$ are
sufficiently far apart to predict correctly. Finally, for coverage functions \cite{feldman2014learning,Sketching}, we achieve $\alpha(n) = 1+\epsilon$.
In Appendix \ref{ksubmodular}, we study comparator-learning $k$-submodular functions
(submodular functions with range in $\{1,\dots, k\}$) in the membership query model, in which the algorithm may ask for labels on examples of its own construction. We show how to learn a $k$-submodular function up to a multiplicative separation of $\alpha$
with sample complexity and running time $O\left(n^{k/\alpha}\right)$.

\subsection{Our Techniques}

Our techniques depart significantly from learning with real-valued labeled examples. When attempting to learn a combinatorial function from cardinal data, rather than ordinal data, the existence of an approximating linear function implies a natural learning algorithm, via a reduction to learning linear separators. In our model, where we are only allowed to make pairwise comparison queries, we require a substantially different algorithm and analysis. At a high level, the existence of an approximating linear function still implies useful structure: once we know that such a function $\vec{w}$ exists and approximates the underlying function $f$ up to an $\alpha(n)$ factor, then given two sets $S$ and $S'$ such that $f(S)$ and $f(S')$ are $\alpha(n)$ apart, we can learn a linear separator that classifies all sets with value less than $f(S)$ as negative and all sets with value greater than $f(S')$ as positive. However, we would like to predict accurately over all random pairs of subsets, not only those whose values are separated by $f(S)$ and $f(S')$. Even more problematic, using only pairwise comparisons, we cannot know if $f(S)$ and $f(S')$ are $\alpha(n)$ apart in the first place.

To surmount this obstacle, we show that we can discretize the range of $f$ using a set of ``landmarks,'' i.e. a sorted set of random examples. For every pair of landmarks $S_i$ and $S_j$, we attempt to solve for a linear separator that classifies a set $S$ as negative if its value is less than $f(S_i)$ and positive if its value is greater than $f(S_j)$. If $f(S_i)$ and $f(S_j)$ are at least $\alpha(n)$ apart, then we are guaranteed that a such a linear separator must exist. Naturally, on a random pair of subsets $T$ and $T'$, the learned comparator simply searches for a linear separator that classifies $T$ as positive and $T'$ as negative, which implies that $f(T) < f(T')$, or vice versa. The key idea which guarantees the correctness of the algorithm is the fact that on one such random query, it is highly unlikely that $f(T)$ and $f(T')$ are $\alpha(n)$ apart, and yet there does not exist a pair of landmarks $S_i$ and $S_j$ such that (1) $f(S_i)$ and $f(S_j)$ sit between $f(T)$ and $f(T')$ and (2) $f(S_i)$ and $f(S_j)$ are also $\alpha(n)$ apart. This is exactly what we need, because if such a pair of landmarks \emph{does} exist, then we will have solved for a linear separator that correctly classifies $T$ and $T'$.

\subsection{Related Work}
Past work has explored the learnability of submodular and related functions
when given access to a set of random examples labeled by the underlying function. Goemans et al.\, showed how to learn an approximation of a submodular function within a multiplicative $\tilde{O}(\sqrt{n})$ factor  \cite{goemans2009approximating} in the membership query model, i.e.\
the queries are selected adaptively by the algorithm.

Balcan and Harvey showed how to efficiently learn a function that approximates the given
submodular factor up to a $\sqrt{n}$ factor on a $1-\epsilon$ fraction of the test inputs, with probability
$1-\delta$, in the supervised learning setting \cite{Balcan11submodular}.
They call this model the \emph{PMAC model} of learning,
where PMAC stands for ``Probably Mostly Approximately Correct,''
due to similarity to the PAC model of learning.
They also show an $\Omega(n^{\frac{1}{3}})$ lower bound in this model.
A later paper by Balcan et al.\ show near tight bounds on the PMAC learnability of
subadditive functions and XOS (fractionally subadditive) functions \cite{balcan2012learning}.

There is also a large body of work on learning submodular functions with additive, rather than multiplicative, guarantees, when the
underlying distribution over subsets of $[n]$ is uniform.
Gupta et al.\ gave an algorithm with runtime $n^{O(\log(1/\delta)/\epsilon^2)}$
which learns an approximation $h$ to a submodular function $f$ such that
with high probability, $|f(x) - h(x)| \leq \epsilon$ \cite{gupta2013privately}.
Feldman et al.\ show an algorithm with runtime $2^{\tilde{O}(1/\epsilon^{4/5})} \cdot n^2$ for approximating
a submodular function with $L_2$ error $\epsilon$ \cite{FeldmanV15}.
Both of these results are accomplished by proving there exist low degree polynomials which approximate submodular functions.

Badanidiyuru et al.\ showed that submodular functions
always have an approximate function with a small sketch \cite{Sketching},
and Iyer et al.\ showed parameterized bounds based on the curvature of the
submodular function (how close the function is to being fully additive)
\cite{iyer2013curvature}.

We conclude this section by reviewing related works on ranking via pairwise comparisons.
Jamieson and Nowak study this problem under the assumption that the $n$ objects they wish to rank are embedded into a $d$-dimensional Euclidean space and that the ranking reflects the objects' relative distance to some fixed point in $\R^d$.
They show an algorithm to learn the rank using $O(d\log n)$ queries on average \cite{jamieson2011active}.
Shah et al.\ study the ranking problem by assuming that the ranking reflects the inherent ``qualities'' of the objects, as defined by a vector $\vec{w}^* \in \R^n$ \cite{ShahBBPRW15}. They work under the standard Bradley-Terry-Luce and Thurstone models, and prove upper and lower bounds on the optimal error when estimating $\vec{w}^*$ in these models.

\section{Preliminaries}

\subsection{Combinatorial Functions}

Throughout this paper, we study different classes of combinatorial functions.
All functions we study are defined over subsets of a ground set $[n]=\{1,\dots,n\}$ and map $2^{[n]}$ to $\R$.
We use $\chi(S)$ to represent the indicator function of the set $S$, so $(\chi(S))_i=1$ if and only if $i\in S$, otherwise
$(\chi(S))_i=0$.

We define three important function classes here and defer the rest of the definitions to their respective sections.

\noindent\textbf{Subadditive functions}.
A function $f$ is \textit{subadditive} if and only if
$f(S\cup S')\leq f(S)+f(S')$, for all $S,S'\subseteq [n]$.
Intuitively, the value of a set is at most the sum of its parts.

\noindent\textbf{Submodular functions}.
A function $f$ is submodular if and only if $f(T\cup\{i\})-f(T)\leq f(S\cup\{i\})-f(S)$ for all $S\subseteq T\subseteq [n]$.
Submodular functions model valuations that satisfy diminishing returns.
Submodularity can also be thought of as the discrete analog of convex functions.

\noindent\textbf{XOS functions}.
A function $f$ is XOS if and only if $f(S)=\max_{j=1\dots k} w_j^T\chi(S)$ where $w_{j} \in \R^n_{\geq 0}$. Alternatively, an XOS function is a MAX of SUM trees. For example, several vacation sites might offer overlapping amenities, and a vacationer might value the amenities differently depending on the site. She will likely then choose the site that has the maximum sum value for the amenities offered there, which means that her valuation function is XOS.

\subsection{Learning Model}

We now define our notion of learnability up to comparisons.
Let $f$ be an unknown function from some class $\mathcal{F}$ (for instance, the class of submodular functions)
and suppose that sets are drawn from some distribution $\D$ over $2^{[n]}$. Moreover, suppose that we have access to a pairwise comparison oracle which, on input $S,S' \in 2^{[n]}$, returns 1 if $f(S) \leq f(S')$ and 0 otherwise. Clearly, we cannot hope to learn $f$ well in absolute terms in this model. Rather, our goal is to produce a hypothesis $g$ such that for most pairs $S,S' \sim D$, either $g$ predicts correctly which of $f(S)$ or $f(S')$ is larger, or $f(S)$ and $f(S')$ are separated by less than a multiplicative $\alpha(n)$ factor. 
We formally define this learning model as follows.

\begin{definition}[comparator-learnable with separation $\alpha(n)$]
A class $\mathcal{F}$ of functions is \textit{comparator-learnable with multiplicative
separation $\alpha(n)$}
if for all $\epsilon,\delta \in (0,1)$ and all $f\in \mathcal{F}$, given a sample of sets over a
distribution $\mathcal{D}$ with size polynomial
in $n,~\frac{1}{\epsilon}$, and $\frac{1}{\delta}$, and given access to a pairwise comparison oracle,
then there exists an algorithm which outputs a pairwise function $g$ with the following guarantee: with probability
at least $1-\delta$, for $S,S' \sim \mathcal{D}$, the probability that $\alpha(n) f(S) \leq f(S')$ yet $g$ incorrectly predicts that $f(S) > f(S')$ is at most $\epsilon$.
\end{definition}

In Section \ref{sec:multiplicativeError}, we present a general algorithm
which can be used to efficiently comparator-learn over a variety of combinatorial function classes 
with separation $\alpha(n)$, where the value of $\alpha(n)$ depends on the complexity of the function class at hand. 
For example, for fully additive functions, $\alpha(n) = 1$, and for submodular functions that are close to being fully additive, i.e. those with small curvature $\kappa$, we have that $\alpha(n) = \frac{1}{1-\kappa}$. Meanwhile, for complex function classes, such as monotone submodular, $\alpha(n) = \sqrt{n}$.
We note that even when given access to real-valued labeled examples, we do not know how to learn a function that approximates a monotone submodular function up to any multiplicative factor better than $\sqrt{n}$. In fact, we prove a nearly-matching lower bound for monotone submodular functions, namely that it is not possible to comparator-learn over this function class with separation $o(n^{1/3}/\log n)$.

Next, in Section \ref{sec:additiveError}, we introduce a related definition: 
comparator-learnable with \emph{additive} separation $\beta$.
In this case, for most pairs $S,S' \sim \mathcal{D}$, we guarantee that the learned comparator either predicts correctly which of $f(S)$ or $f(S')$ is larger, or $f(S)$ and $f(S')$ are separated by less than an additive $\beta$ factor. 
Finally, in Section~\ref{sec:otherCombinatorial}
we show that for certain highly structured function classes, we are able to learn with no separation. In other words, the learned comparator will predict accurately on a large fraction
of \emph{all} pairs, not just pairs which are sufficiently far apart.

\section{General Framework for Learning with Comparisons}\label{sec:multiplicativeError}

In this section, we present a general algorithm for learning combinatorial functions
up to pairwise comparisons. 
We guarantee that for an underlying combinatorial function $f$, with high probability our algorithm outputs a hypothesis $g$, where given a random $S,S'$, the probability that 
$f(S)$ and $f(S')$ differ by a large multiplicative factor yet $g$ predicts incorrectly is low.

We describe our algorithm for a large family of general function classes, each of which has the property that any function in the class can be approximated by the $p$-th root of linear function up to an $\alpha(n)$ factor. We then instantiate this algorithm for many classes of combinatorial functions. To make this notion more concrete, we define the following characterization of a class of functions.

\begin{definition}[$\alpha(n)$-approximable]\label{def:approximable}
A class $\mathcal{F}$ of set functions is \textit{$\alpha(n)$-approximable} 
if there exists a $p\in\mathbb{R}$
such that for all $f\in \mathcal{F}$, there exists a vector $\vec{w}_f\in\mathbb{R}^n$ such that for all $S \subseteq [n]$,
$f(S)\leq (\vec{w}_f \cdot\chi(S))^p\leq \alpha(n) f(S)$.
\end{definition}

\noindent \textbf{High Level Description of the Algorithm and Analysis.}\begin{figure}
	{\includegraphics[width=\textwidth]{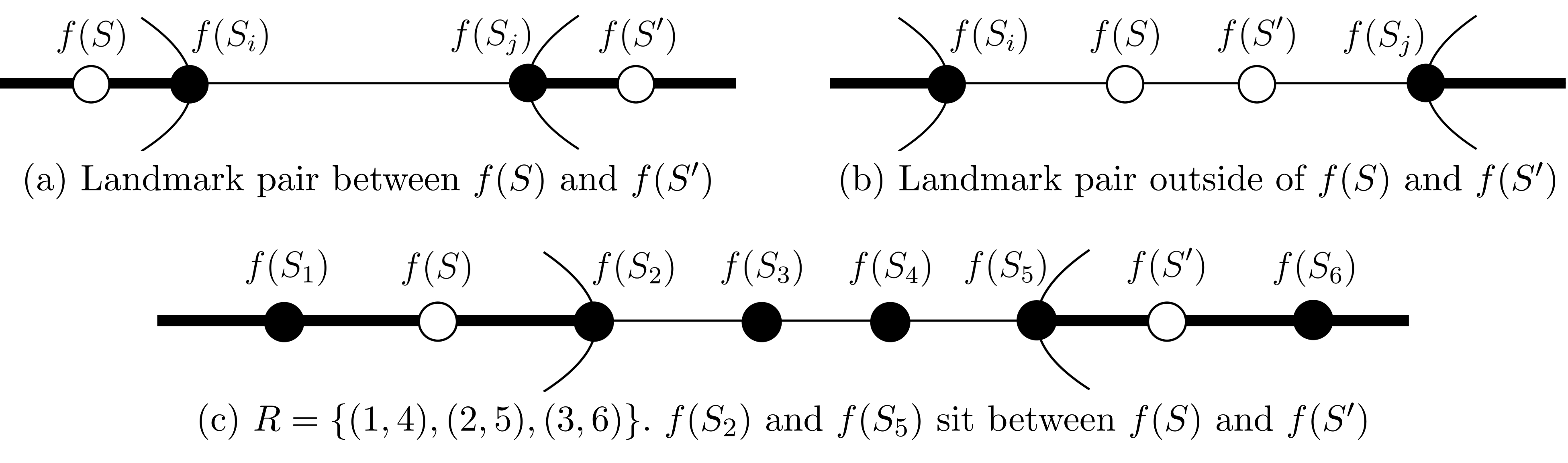}}
    \caption{Illustration of landmark pairs in Algorithm~\ref{alg:active_learning}.}
    \label{fig:landmark_regions}
\end{figure}
As a crucial first step in our algorithm design, we show that if the underlying function $f$ is $\alpha(n)$-approximable, then
there exists a vector $\vec{w}_f \in \R^n$ such that for any sets $S_i$ and $S_j$ in our sample, if $\alpha(n)f(S_i)<f(S_j)$, then $ \vec{w}_f\cdot\chi(S_i)< \vec{w}_f\cdot\chi(S_j)$. As one would expect, this is simply the vector $\vec{w}_f$ referred to in Definition~\ref{def:approximable}.
Taking one step further, this means that there exists a linear separator such that if $f(S) < f(S_i) < f(S_j) < f(S')$ for two sets $S$ and $S'$, then $S$ is labeled as negative and $S'$ is labeled as positive by the linear separator. This situation is illustrated by Figure \ref{fig:landmark_regions}a. This fact alone will not be enough when designing our comparator learning algorithm. After all, we want to predict accurately on all random pairs, not just those that span $f(S_i)$ and $f(S_j)$. Moreover, without real-valued labels, it is impossible to know if $f(S_i)$ and $f(S_j)$ are $\alpha(n)$ apart to begin with.

This suggests the following algorithm. We begin by discretizing the range of $f$ using a series of ``landmarks.'' This is simply an initial random sample, which we sort according to $f$. Since $f$ is $\alpha(n)$-approximable, we know that if two landmarks $S_i$ and $S_j$ are at least $\alpha(n)$ apart, then there exists a weight vector $\vec{w}_{ij} \in \R^n$ and threshold $\theta_{ij} \in \R$ that classifies a set $S$ as negative if $f(S) < f(S_i)$ and positive if $f(S) > f(S_j)$. Namely, $\vec{w}_{ij}=\vec{w}_f$ and $\theta_{ij}=\frac{1}{2}\left[\vec{w}_f\cdot\chi(S_i)+\vec{w}_f\cdot\chi(S_j)\right]$.

Therefore, we attempt to learn $\vec{w}_{ij}, \theta_{ij}$ for every landmark pair $S_i$ and $S_j$ such that $f(S_i) \leq f(S_j)$. From pairwise comparison queries alone, we cannot know if $f(S_i)$ and $f(S_j)$ are $\alpha(n)$ apart, which would mean that $(w_{ij}, \theta_{ij})$ has no training error, so we must attempt to solve for such a linear separator for all landmark pairs. Using a set $R$, we save the indices of the landmark pairs whose corresponding linear separator has no training error. Crucially, we guarantee that it is highly unlikely, on a random query $(S,S')$, that $f(S)$ and $f(S')$ are $\alpha(n)$ apart, and yet there does not exist a pair of landmarks $S_i$ and $S_j$ such that (1) $f(S_i)$ and $f(S_j)$ fall in between $f(S)$ and $f(S')$ and (2) $f(S_i)$ and $f(S_j)$ are also separated by an $\alpha(n)$ factor. This guarantee is illustrated by Figure 1c. If such a pair of landmarks $S_i$ and $S_j$ do exist, we can be confident that we solve for a linear separator that correctly classifies $S$ and $S'$.

Ultimately, on input $(S,S')$, the learned comparator will search for a pair $(i,j)$ in $R$ such that the corresponding linear separator $(\vec{w}_{ij}, \theta_{ij})$ classifies $S$ as positive and $S'$ as negative, or vice versa, in which case $f(S) < f(S')$ or the opposite, respectively. However, we cannot guarantee that this will work as planned just yet. After all, suppose that on a random query $(S,S')$, $f(S)$ and $f(S')$ are $\alpha(n)$ apart but there is some pair of landmarks $S_i$ and $S_j$ such that $f(S)$ and $f(S')$ fall in between $f(S_i)$ and $f(S_j)$, instead of the other way around. This event is illustrated in Figure \ref{fig:landmark_regions}b. In this situation, we can guarantee nothing about how $S$ and $S'$ will be classified. To avoid this case, we remove all but the minimal pairs in $R$. In other words, if there exists $(i,j)$ and $(i', j')$ in $R$ such that $f(S_i) \leq f(S_{i'}) < f(S_{j'}) \leq f(S_j)$, then we remove $(i,j)$ from $R$. Therefore, so long as there do exist two landmarks falling between $f(S)$ and $f(S')$ that are $\alpha(n)$ apart, we can be guaranteed that any bad landmark pair $(S_i, S_j)$, as illustrated by Figure \ref{fig:landmark_regions}b, is removed, and the learned comparator never considers $(\vec{w}_{ij}, \theta_{ij})$ when making a prediction.

Now we state the algorithm formally, and in Corollary~\ref{cor:structure_classes}, show that it can be instantiated for many different combinatorial function classes, achieving polynomial sample complexity as well as efficient runtime.

\begin{algorithm}
\DontPrintSemicolon
\KwIn{Sample $\mathcal{S}\sim\mathcal{D}$ of size $\tilde O(\frac{n}{\epsilon^3})$,
pairwise comparison oracle}
\KwOut{Function $g:2^{[n]}\times 2^{[n]}\rightarrow\{0,1\}$}
\begin{enumerate}
\item Remove $m=\frac{2}{\epsilon}\log(\frac{1}{\epsilon\delta})$ samples uniformly at random from $\mathcal{S}$. Label this set $\mathcal{S}_1$,\\
and label $\mathcal{S}\setminus\mathcal{S}_1$ as $\mathcal{S}_2$.
\item Sort $\mathcal{S}_1$ into $f(S_1)\leq \cdots\leq f(S_m)$. \label{step:sort}
\item Sort $\sample_2$ into the sets $\sample_{ij} = \{S \ | \ S \in \sample_2 \text{ and } f(S) \leq f(S_i) \text{ or } f(S) \geq f(S_j)\}$ for all\\ $S_i, S_j \in \sample_1$ such that $i < j$.
\item For each $S_i,S_j\in \mathcal{S}_1$ (wlog $i<j$), attempt to find $\theta_{ij}\in \mathbb{R}$ and 
$\vec{w}_{ij}\in \mathbb{R}^n$ such that for all $S\in \sample_{ij}$, $f(S)<f(S_i)\implies \vec{w}_{ij} \cdot \chi(S)<\theta_{ij}$ 
and $f(S_j)<f(S)\implies \vec{w}_{ij}\cdot \chi(S)>\theta_{ij}$.\label{step:train}
\item If the previous step is successful, put $(i,j)$ into $R$.
\item Remove all but the ``minimal'' pairs in $R$: if there exists
$(i,j),(i',j')\in R$ such
that\\ $i\leq i'$ and $j\geq j'$, remove $(i,j)$ from $R$. \label{step:minimal}
\item Define the function $g(S,S')$ as follows. Return 1 if $\exists (i,j)\in R$ such
 that $\vec{w}_{ij}\cdot\chi(S)<\theta_{ij}<\vec{w}_{ij}\cdot\chi(S')$.
Otherwise, return 0.
\end{enumerate}
\caption{Algorithm for comparator-learning combinatorial functions.}
\label{alg:active_learning}
\end{algorithm}

\begin{theorem} \label{thm:general}
Let $\mathcal{F}$ be an $\alpha(n)$-approximable class.
Then $\mathcal{F}$ is comparator-learnable with separation $\alpha(n)$, using Algorithm \ref{alg:active_learning}.
\end{theorem}

\begin{proof}
First, we show there exists a $\vec{w} \in \R^n$ such that if $\alpha(n) f(S) < f(S')$, then $\vec{w} \cdot \chi(S) < \vec{w} \cdot \chi(S')$.
Since $f$ is from an $\alpha(n)$-approximable class, 
we know there exists $\hat{f}(S) = (\vec{w}_f \cdot\chi(S))^p$ such that
$f(S)\leq \hat{f}(S)\leq \alpha(n) f(S)$ for all $S$.  This
implies that if $\alpha(n) f(S) < f(S')$, then $\hat{f}(S) < \hat{f}(S')$, which in
turn implies that  $(\vec{w}_f \cdot\chi(S))^p < (\vec{w}_f \cdot\chi(S'))^p$. Finally, this means that  
$\vec{w}_f \cdot\chi(S) < \vec{w}_f \cdot\chi(S')$.

Now we prove that the learned comparator has low error by splitting the analysis into two parts.
First, we show that on a random pair $S,S'$, it is unlikely that $f(S)$ and $f(S')$ are an $\alpha(n)$ factor apart and yet there is no landmark pair
$S_i,S_j \in \sample_1$ such that (1) $f(S_i)$ and $f(S_j)$ fall in between $f(S)$ and $f(S')$ (i.e. $f(S) < f(S_i) < f(S_j) < f(S')$) and (2) $f(S_i)$ and $f(S_j)$ are separated by at least an $\alpha(n)$ factor. This is exactly what we need, because if such a pair $S_i,S_j$ does exist, then during Step~\ref{step:train} of Algorithm~\ref{alg:active_learning}, we will have solved for a linear separator $(\vec{w}_{ij}, \theta_{ij})$ that will label $S$ as negative and $S'$ as positive, with high probability. 
We prove this formally in Claim~\ref{claim:buckets}.

The only case where one of the linear separators $(\vec{w}_{ij}, \theta_{ij})$ would fail to label $S$ or $S'$ correctly is if $S$ or $S'$ contribute to the learning error inherent to learning linear separators. To handle this case, we show that on a random pair $S,S' \sim \mathcal{D}$, the probability that $f(S)$ and $f(S')$ are at least an $\alpha(n)$ factor apart and yet some linear separator learned during Step~\ref{step:train} mislabels $S$ or $S'$ is low. We prove this in Claim~\ref{claim:lin_sep_error}.

We combine Claim~\ref{claim:buckets} and Claim~\ref{claim:lin_sep_error} to prove the correctness of Algorithm~\ref{alg:active_learning} in the following way. We claim that with probability at least $1-\delta$, on a random pair $S,S' \sim \mathcal{D}$, the probability that $\alpha(n) f(S) < f(S')$ yet the learned comparator $g$ predicts that $f(S') < f(S)$ is low. This will happen whenever there exists a pair $(i,j) \in R$ such that $\vec{w}_{ij} \cdot \chi(S) > \theta_{ij}$ and $\vec{w}_{ij} \cdot \chi(S') < \theta_{ij}$. In particular, we want to bound \[\Pr_{S,S' \sim \mathcal{D}}[\alpha(n) f(S) < f(S') \text{ and } \exists(i,j) \in R \text{ such that } \vec{w}_{ij} \cdot \chi(S) > \theta_{ij} \text{ and }\vec{w}_{ij} \cdot \chi(S') < \theta_{ij}].\] To analyze this probability, we partition the pairs $(i,j)\in R$ into two sets:
\begin{align*}
&R_1 = \{(i,j) \ | \ f(S_i) < f(S) \text{ and } f(S') < f(S_j)\} \text{ and}\\
&R_2 = \{(i,j) \ | \ f(S)\leq f(S_i) \text{ or } f(S_j) \leq f(S')\}.\end{align*}

Clearly, the probability that $g$ predicts incorrectly on $S,S' \sim \mathcal{D}$ due to a pair $(i,j) \in R$ is simply the probability that $g$ predicts incorrectly due to a pair $(i,j) \in R_1$ or a pair $(i,j) \in R_2$. With this in mind, we first analyze \[\Pr_{S,S' \sim \mathcal{D}}[\alpha(n) f(S) < f(S') \text{ and } \exists(i,j) \in R_1 \text{ such that } \vec{w}_{ij} \cdot \chi(S) > \theta_{ij} \text{ and }\vec{w}_{ij} \cdot \chi(S') < \theta_{ij}].\] Recall that in Step~\ref{step:minimal} of Algorithm~\ref{alg:active_learning}, all non-minimal pairs were from $R$. This means that if $(i,j) \in R_1$, then it must be minimal, so there must not exist $S_{i'}, S_{j'} \in \sample_1$ such that $\alpha(n) f(S) < \alpha(n) f(S_{i'}) < f(S_{j'}) < f(S).$ After all, if such a pair $S_{i'}, S_{j'}$ did exist, then we would have obtained the linear separator $(\vec{w}_{i'j'}, \theta_{i'j'})$ in Step~\ref{step:train}, and $(i,j)$ would have no longer been minimal. Therefore, the probability that $g$ predicts incorrectly due to a pair $(i,j) \in R_1$ is simply \[\Pr_{S,S' \sim \D}\left[\alpha(n) f(S) < f(S') \mbox{ and } \not\exists S_i,S_j \in \sample_1:
\alpha(n) f(S) \leq \alpha(n) f(S_i) < f(S_j) \leq f(S')\right].\] This is exactly the probability we bound in Claim~\ref{claim:buckets}, which means that if we set $\epsilon' = \frac{\epsilon}{2}$ and $\delta' = \frac{\delta}{2}$, then with probability at most $\frac{\delta}{2},$ \[\Pr_{S,S' \sim \mathcal{D}}[\alpha(n) f(S) < f(S') \text{ and } \exists(i,j) \in R_1 \text{ s.t. } \vec{w}_{ij} \cdot \chi(S) > \theta_{ij} \text{ and }\vec{w}_{ij} \cdot \chi(S') < \theta_{ij}] > \frac{\epsilon}{2}.\]

Meanwhile, whenever $g$ predicts incorrectly due to a pair $(i,j) \in R_2$, it means that $f(S) \leq f(S_i)$ and $\vec{w}_{ij} \cdot \chi(S) > \theta_{ij}$ or $f(S_j) \leq f(S')$ and $\vec{w}_{ij} \cdot \chi(S') < \theta_{ij}$. In other words, $S$ or $S'$ contributes to the learning error of $(\vec{w}_{ij}, \theta_{ij})$. This is the probability we bound in Claim~\ref{claim:lin_sep_error}, which means that if we set $\epsilon' = \frac{\epsilon}{2}$ and $\delta' = \frac{\delta}{2}$, we have that with probability at most $\frac{\delta}{2},$ \[\Pr_{S,S' \sim \mathcal{D}}[\alpha(n) f(S) < f(S') \text{ and } \exists(i,j) \in R_2 \text{ s.t. } \vec{w}_{ij} \cdot \chi(S) > \theta_{ij} \text{ and }\vec{w}_{ij} \cdot \chi(S') < \theta_{ij}] > \frac{\epsilon}{2}.\]

Putting these bounds together, we have that with probability at most $\delta,$ \[\Pr_{S,S' \sim \mathcal{D}}[\alpha(n) f(S) < f(S') \text{ and } \exists(i,j) \in R \text{ such that } \vec{w}_{ij} \cdot \chi(S) > \theta_{ij} \text{ and }\vec{w}_{ij} \cdot \chi(S') < \theta_{ij}] > \epsilon.\] Therefore, with probability at least $1-\delta$, the probability that $f(S)\alpha(n) < f(S')$ and $g$ predicts incorrectly is at most $\epsilon$.

\end{proof}
Now we prove the first claim, which guarantees that with probability at least $1-\delta$, at most an $\epsilon$ density of pairs $S, S' \sim \mathcal{D}$ have the property that $\alpha(n) f(S) < f(S')$ and yet there does not exist 
a landmark pair 
$S_i, S_j \in \sample_1$ such that $f(S_i)$ and $f(S_j)$ are separated by a $\alpha(n)$ multiplicative factor and $f(S_i)$ and $f(S_j)$ sit between $f(S)$ and $f(S')$.

\begin{claim}\label{claim:buckets}
A sample size $m = O\left(\frac{1}{\epsilon}\log\frac{1}{\epsilon\delta}\right)$ is
sufficient so that with probability at least $1-\delta$, $\sample_1 = \{S_1, \dots, S_m\}$ has the property that
$$
\Pr_{S,S' \sim \D}\left[\alpha(n) f(S) < f(S') \mbox{ and } \not\exists S_i,S_j \in \sample_1:
\alpha(n) f(S) \leq \alpha(n) f(S_i) < f(S_j) \leq f(S')\right]\leq \epsilon.
$$
\end{claim}

\begin{proof}
First, note that the distribution $\D$ induces a distribution over values
$f(S) \in \R_{\geq 0}$.  Partition $\R_{\geq 0}$ into $1/\epsilon$ buckets of
probability mass $\epsilon$.
 If $m =
O\left(\frac{1}{\epsilon}\log\frac{1}{\epsilon\delta}\right)$,
then with probability at least $1-\delta$, our sample $\sample_1 = \{S_1, S_2, \ldots, S_m\}$
contains at least one set from every
bucket. Let us assume below this is indeed the case.

Now, define two buckets $b_i$ and $b_j$ to be {\em close} if all points in
$b_i$ are within an $\alpha(n)$ factor of all points in $b_j$.  Define
the buckets to be {\em far} if all points in $b_i$ have a gap
of size greater than $\alpha(n)$ with all points in $b_j$.  Otherwise (some points in
$b_i$ are close to some points in $b_j$ and some are far), say that the
two buckets are {\em in conflict}.
Say that $S$ is in bucket $b_{i'}$ and $S'$ is in bucket $b_{j'}$
(WLOG, $i'\leq j'$). If $b_{i'+1}$ and $b_{j'-1}$ are far, then
we have the desired $S_i$ and $S_j$ (from buckets $b_{i'+1}$ and $b_{j'-1}$).
Also, if $b_{i'}$ and $b_{j'}$ are close buckets,
then we are fine (the bad event in the statement does not occur)
since $S$ and $S'$ must have values that do not
differ by a $\alpha(n)$ factor.  
The final case is if $b_{i'+1}$ is not far from $b_{j'-1}$ and
$b_{i'}$ is not close to $b_{j'}$. This is captured by the case where:
$b_{i'+1}$ is in conflict with $b_{j'-1}$, $b_{i'+1}$ is in
conflict with $b_{j'}$, $b_{i'}$ is in conflict with $b_{j'-1}$, or $b_{i'}$ is in conflict with $b_{j'}$.

Now we show the fraction of pairs of buckets in conflict is low.
If, say, $b_{i'}$ and $b_{j'}$ are in conflict,
then $b_{i'-1}$ and $b_{j'+1}$ must be far and, assuming $i<j-1$,
$b_{i'+1}$ and $b_{j'-1}$ must be close. This implies there are at most
$1/\epsilon$ pairs of buckets in conflict because for every bucket $b_i$, there is at most one bucket $b_j$ it can be in conflict with. Since there are at most $1/\epsilon$ pairs of buckets in
conflict but $1/\epsilon^2$ pairs of buckets overall, this implies
the probability that two buckets are in conflict is $O(\epsilon)$.
\end{proof}
We now provide a proof sketch of Claim~\ref{claim:lin_sep_error}. The full proof can be found in Appendix~\ref{app:multiplicativeError}.

\begin{claim}\label{claim:lin_sep_error} 
Let $P_1 = \Pr_{S \sim \D}[\exists(i,j) \in R \text{ such that } f(S) \leq f(S_i) \mbox{ yet } \vec{w}_{ij} \cdot \chi(S) > \theta_{ij}]$ and
$P_2 = \Pr_{S \sim \D}[\exists(i,j) \in R \text{ such that } f(S) \geq f(S_j)\mbox{ yet } \vec{w}_{ij} \cdot \chi(S) < \theta_{ij}].$
A sample of size \[|\sample_2| = O\left(\frac{m^2}{\epsilon}\left[n \log \frac{m^2}{\epsilon} + \log\frac{1}{\delta}\right]\right)\] is sufficient so that with probability at least $1-\delta$, 
$P_1+P_2 < \epsilon.$
\end{claim}

\begin{proof}[Proof sketch.]
For each $(i,j)$, we define a class of loss functions $\mathcal{L}_{ij}=
\{L_{(\vec{w}_{ij},\theta_{ij})} \ | \ \vec{w}_{ij} \in \R^n, \theta_{ij} \in \R\},$ where for any $(\vec{w}_{ij}, \theta_{ij})$ and any $S \subseteq [n]$, $L_{(\vec{w}_{ij}, \theta_{ij})}(S)$ is equal to 1 if
$f(S)\leq f(S_i)$ and $\vec{w}_{ij}\cdot\chi(S)>\theta_{ij}$ or if $f(S) \geq f(S_j)$ and $\vec{w}_{ij} \cdot \chi(S) < \theta$, and 0 otherwise.
It is straightforward to show the VC dimension of $\mathcal{L}_{ij}$ is the same as the 
VC dimension of the class of linear separators over $\R^n$, which is $n+1$.
We know that for any pair $S_i$ and $S_j$ such that $\alpha(n)f(S_i)\leq f(S_j)$, the empirical risk minimizer of
$\mathcal{L}_{ij}$ will have zero loss over $\mathcal{S}_2$.
Therefore, by standard VC dimension bounds,
if $|\sample_2| = O(\frac{m^2}{\epsilon}[n \log \frac{m^2}{\epsilon} + \log\frac{1}{\delta}])$, then with probability at least $1-\delta$, the error of $h^*$ over $\D$ is at most $\frac{\epsilon}{m^2}$. 
Now we union bound over all $m^2$ pairs $(i,j)$ on which 
Algorithm \ref{alg:active_learning} attempts to solve for a linear threshold 
$(\vec{w}_{ij}, \theta_{ij})$ and thereby achieve an overall error of $\epsilon$.
\end{proof}
These claims complete the proof of Theorem \ref{thm:general}.
By using structural results for submodular, XOS, and subadditive functions, as well as submodular functions with bounded curvature and XOS functions with a polynomial number of SUM trees, 
compiled from \cite{Balcan11submodular,balcan2012learning,iyer2013curvature}
we immediately obtain the following corollary to Theorem \ref{thm:general} 
(the formal proof is in Appendix \ref{app:multiplicativeError}).

\begin{cor} \label{cor:structure_classes}
The following statements are true:
\begin{enumerate}
\item The class of monotone submodular functions is efficiently comparator-learnable with separation $\sqrt{n}$.
\item The class of XOS functions is efficiently comparator-learnable with separation $O(\sqrt{n})$.
\item The class of monotone subadditive functions is efficiently comparator-learnable 
with separation $\sqrt{n}\log{n}$.
\item The class of submodular functions with curvature $\kappa$ is efficiently comparator-learnable with
separation $\min\left\{\sqrt{n},\frac{1}{1-\kappa}\right\}$.\label{cor:curvature}
\item For any $\xi>0$, the class of XOS functions with $R$ SUM trees is comparator-learnable with separation $R^{\xi}$. In this case, the sample complexity and running time of Algorithm~\ref{alg:active_learning} are both polynomial in $n^{\frac{1}{\xi}}$.\label{cor:poly_sum_trees}
\end{enumerate}
\end{cor}

\subsection{Lower Bounds} \label{sec:lower-bound}

Now, we show that the class of non-negative, monotone, submodular functions is not comparator-learnable with separation $o(n^{1/3}/ \log n)$. This lower bound nearly matches our upper bound from Corollary \ref{cor:structure_classes}.1. To prove this result, we use a special family of matroid rank functions, which form a subset of the class of monotone submodular functions, presented in \cite{Balcan11submodular} and described as follows.

\begin{theorem}[Theorem 7 in \cite{Balcan11submodular}]\label{thm:bh_lb}
For any $k\geq 8$ with $k=2^{o(n^{1/3})}$, there exists a family of sets $\mathcal{A}\subseteq 2^{[n]}$ and a family of matroids $\mathcal{M}= \{\mathcal{M}_{\mathcal{B}}:\mathcal{B}\subseteq\mathcal{A}\}$ with the following properties.
\begin{itemize}
\item $|\mathcal{A}|=k$ and for each $A\in\mathcal{A}$, $|A|=n^{1/3}$.
\item For each $\mathcal{B}\subseteq\mathcal{A}$ and $A\in\mathcal{A}$,
\begin{equation*}
\text{rank}_{\mathcal{M}_{\mathcal{B}}}(A)=
\begin{cases}
8\log k & \text{if }A\in\mathcal{B}  \\
|A| & \text{otherwise.}
\end{cases}
\end{equation*}
\end{itemize}
\end{theorem}

We now formally present our lower bound.

\begin{theorem}\label{thm:sub_lower}
Let $\mathcal{ALG}$ be an arbitrary learning algorithm that uses only a polynomial number of training samples drawn i.i.d. from the underlying distribution and produces a predictor $g$. There exists a distribution $\mathcal{D}$ and a submodular target function $f^*$ such that, with probability at least $1/25$ (over the draw of the training samples),
\[\Pr_{S,S' \sim \mathcal{D}} \left[\alpha(n) f(S) \leq f(S') \text{ and } g \text{ predicts that } f(S) > f(S')\right] \geq \frac{1}{25},\] where $\alpha(n) = \Omega(n^{1/3}/\log n)$.
\end{theorem}

\begin{proof}
As in \cite{Balcan11submodular}, we use the family of matroids presented in Theorem~\ref{thm:bh_lb} to show that for a super-polynomial sized set of $k$ points in $\{0,1\}^n$, and for any partition of those points into \textsc{High} and \textsc{Low}, we can construct a matroid where the points labeled \textsc{High} have rank $r_{high}$ and the points labeled \textsc{Low} have rank $r_{low}$, and that $r_{high}/r_{low} = \tilde{\Omega}(n^{1/3})$. We prove that this implies hardness for comparator-learning over the uniform distribution on these $k$ points from any polynomial-sized sample.

Formally, we use the probabilistic method to prove the existence of the $f^*$ referred to in the theorem statement. To this end, suppose that $\mathcal{ALG}$ uses $\ell \leq n^c$ training examples for some constant $c$. To construct a hard family of submodular functions, we apply Theorem~\ref{thm:bh_lb} with $k = 2^t$ where $t = c\log n + 3$. Let $\mathcal{A}$ and $\mathcal{M}$ be the families that are guaranteed to exist, and let $\text{rank}_{M_{\mathcal{B}}}$ be the rank function of a matroid $M_{\mathcal{B}} \in \mathcal{M}$. Let the underlying distribution $\mathcal{D}$ on $2^{[n]}$ be the uniform distribution on $\mathcal{A}$.

Assume that $\mathcal{ALG}$ uses a set $\mathcal{S}$ of $\ell$ training examples. For any $S,S' \in \mathcal{A} \setminus \mathcal{S}$ such that $S \not= S'$, we claim that the algorithm $\mathcal{ALG}$ has no information about how $f^*(S)$ compares to $f^*(S')$, for any target function $f^* = \text{rank}_{M_{\mathcal{B}}}$ where $M_{\mathcal{B}} \in \mathcal{M}$. After all, by Theorem~\ref{thm:bh_lb}, for a fixed labeling of $\sample$ by the values $r_{low}$ and $r_{high}$, there exist exactly $2^{|\mathcal{A}\setminus \sample|}$ matroids in $\mathcal{M}$ whose respective rank functions label $\sample$ according to this fixed labeling. Moreover, for any partition of the points in $\mathcal{A} \setminus \mathcal{S}$ into \textsc{High} and \textsc{Low}, there exists exactly one matroid among those $2^{|\mathcal{A}\setminus \sample|}$ matroids such that the points labeled \textsc{High} have rank $r_{high}$ and the points labeled \textsc{Low} have rank $r_{low}$. In other words, if $M_{\mathcal{B}}$ is drawn uniformly at random from $\mathcal{M}$ and $f^* = \text{rank}_{M_{\mathcal{B}}}$, then the conditional distribution of $f^*(S)$ given $\mathcal{S}$ is uniform in $\{r_{low}, r_{high}\}$.

The set of non-training examples has measure $1-2^{-t+\log \ell}$, and as we have seen, in expectation, half of the non-training sets will have rank $r_{high}$ and half will have rank $r_{low}$. Therefore,\[E_{f^*, \sample} \left[\Pr_{S,S' \sim \mathcal{D}} \left[S, S' \not\in \sample\text{, }f^*(S) = r_{low}, \text{ and } f^*(S') = r_{high}\right]\right] = \left(\frac{1-2^{-t+\log \ell}}{2}\right)^2.\]

Moreover, for $S,S' \not\in\sample$, due to the uniform conditional distribution of $f^*(S)$ and $f^*(S')$ given $\sample$, $\mathcal{ALG}$ cannot determine whether $f^*(S) \geq f^*(S')$ or vice versa better than randomly guessing between the two alternatives. Therefore \begin{align*}
&\text{ }E_{f^*, \mathcal{S}} \left[ \Pr_{S,S' \sim \mathcal{D}} \left[\alpha(n) f^*(S) \leq f^*(S') \text{ and } g \text{ predicts that } f^*(S) > f^*(S')\right]\right]\\
\geq &\text{ }E_{f^*, \mathcal{S}} \left[ \Pr_{S,S' \sim \mathcal{D}} \left[S, S' \not\in \sample\text{, }f^*(S) = r_{low} \text{, } f^*(S') = r_{high} \text{, and } g \text{ predicts that } f^*(S) > f^*(S')\right]\right]\\
\geq &\text{ }\frac{1}{2}\left(\frac{1-2^{-t+\log \ell}}{2}\right)^2\\
\geq &\text{ }\frac{49}{512}.\end{align*}

Therefore, there exists a rank function $f^*$ such that \[E_{\sample}\left[ \Pr_{S,S' \sim \mathcal{D}} \left[\alpha(n) f^*(S) \leq f^*(S') \text{ and } g \text{ predicts that } f^*(S) > f^*(S')\right]\right] \geq \frac{49}{512}.\] We claim that this means that for this fixed $f^*$, \[\Pr_{\sample}\left[ \Pr_{S,S' \sim \mathcal{D}} \left[\alpha(n) f(S) \leq f(S') \text{ and } g \text{ predicts that } f(S) > f(S')\right]\geq \frac{1}{25}\right] \geq \frac{1}{25}.\] After all, suppose not, so \[\Pr_{\sample}\left[ \Pr_{S,S' \sim \mathcal{D}} \left[\alpha(n) f(S) \leq f(S') \text{ and } g \text{ predicts that } f(S) > f(S')\right]\geq \frac{1}{25}\right] < \frac{1}{25}.\] Then setting $X := \Pr_{S,S' \sim \mathcal{D}} \left[\alpha(n) f(S) \leq f(S') \text{ and } g \text{ predicts that } f(S) > f(S')\right]$, we have that \[E_{\sample}[X] \leq 1 \cdot \Pr\left[X \geq \frac{1}{25}\right] + \frac{1}{25}\cdot \Pr\left[X < \frac{1}{25}\right] < 1 \cdot \frac{1}{25} + \frac{1}{25}\cdot 1 = \frac{2}{25} < \frac{49}{512}.\] Of course, this is a contradiction, so the theorem statement holds.
\end{proof}

By a similar argument, we show that the class of XOS functions is not comparator-learnable with separation $o(\sqrt{n}/ \log n)$. The proof of Theorem~\ref{thm:XOS_lower} makes use of a special family of XOS functions presented in \cite{balcan2012learning} and follows the same logic as in the proof of Theorem~\ref{thm:sub_lower}. It can be found in Appendix~\ref{app:multiplicativeError}.

\begin{theorem}\label{thm:XOS_lower}
Let $\mathcal{ALG}$ be an arbitrary learning algorithm that uses only a polynomial number of training samples drawn i.i.d. from the underlying distribution and produces a predictor $g$. There exists a distribution $\mathcal{D}$ and an XOS target function $f^*$ such that, with probability at least $1/25$ (over the draw of the training samples),
\[\Pr_{S,S' \sim \mathcal{D}} \left[\alpha(n) f(S) \leq f(S') \text{ and } g \text{ predicts that } f(S) > f(S')\right] \geq \frac{1}{25},\] where $\alpha(n) = \Omega(\sqrt{n}/\log n)$.
\end{theorem}

We also show a lower bound parameterized by the curvature of a submodular function.

\begin{theorem}\label{thm:curvature_lower}
Let $\mathcal{ALG}$ be an arbitrary learning algorithm that uses only a polynomial number of training samples drawn i.i.d. from the underlying distribution and produces a predictor $g$. There exists a distribution $\mathcal{D}$ and a submodular target function $f^*$ with curvature $\kappa$ (possibly known to
the algorithm) such that, with probability at least $1/25$ (over the draw of the training samples),
\[\Pr_{S,S' \sim \mathcal{D}} \left[\alpha(n) f(S) \leq f(S') \text{ and } g \text{ predicts that } f(S) > f(S')\right] \geq \frac{1}{25},\] where 
$\alpha(n) = \frac{n^{1/3}}{O(\kappa\log n)+(1-\kappa)n^{1/3}}$.
\end{theorem}

\begin{proof}
Given a submodular function $f$ with curvature 1, we may convert it
to a submodular function with curvature $\kappa$ by setting
$f_{\kappa}(X)=\kappa f(X)+(1-\kappa)|X|$. This idea allows us to easily modify
the proof of Theorem \ref{thm:sub_lower}.

Suppose that $\mathcal{ALG}$ uses $\ell \leq n^c$ training examples for some constant $c$. We apply Theorem~\ref{thm:bh_lb} with $k = 2^t$ where $t = c\log n + 3$. Let $\mathcal{A}$ and $\mathcal{M}$ be the families that are guaranteed to exist.
Again, $\text{rank}_{M_{\mathcal{B}}}$ denotes the rank function of a matroid 
$M_{\mathcal{B}} \in \mathcal{M}$. Now we define 
$\text{rank}^{\kappa}_{M_{\mathcal{B}}}(A)=\kappa\cdot\text{rank}_{M_{\mathcal{B}}}(A)+
(1-\kappa)|A|$.
Let the underlying distribution $\mathcal{D}$ on $2^{[n]}$ be the uniform distribution on $\mathcal{A}$.
Then for all $\mathcal{B}\subseteq\mathcal{A}$ and $A\in\mathcal{A}$, if $A\in\mathcal{B}$, then 
$\text{rank}^{\kappa}_{M_{\mathcal{B}}}(A)=8\kappa\log k+(1-\kappa)|A|
= 8\kappa\log k+(1-\kappa)n^{1/3}$,
and if $A\notin\mathcal{B}$, then 
$\text{rank}^{\kappa}_{M_{\mathcal{B}}}(A)=\kappa n^{1/3}+(1-\kappa)|A|=n^{1/3}$.
Therefore, $r_{high}/r_{low}= \frac{n^{1/3}}{8\kappa\log k+(1-\kappa)n^{1/3}}$.

As in the previous proof, $\mathcal{ALG}$ has no information about pairs
of sets which are not in the training set.
The rest of the proof is similar to the proof of Theorem \ref{thm:sub_lower}.
\end{proof}

\section{Additive Separation Analysis}\label{sec:additiveError}

Let $f$ be a monotone submodular function with range in $[0,1]$ and fix $\mathcal{D}$ to be the uniform distribution over the $n$-dimensional boolean cube. A slight variant on Algorithm 1 allows us to learn a comparator $g$ which, on input $S,S' \subseteq [n]$, returns whether $f(S)$ is greater than $f(S')$ or vice versa whenever $f(S)$ and $f(S')$ differ by a sufficiently large additive factor $\beta$, rather than a multiplicative factor as in Section~\ref{sec:multiplicativeError}. In this case, we say that $f$ is \emph{comparator-learnable with additive separation $\beta$}.

This result relies on key insights into the Fourier spectrum of submodular functions. In particular, we use the fact that any monotone submodular function $f$ with range in $[0,1]$ is $\gamma$-close to a polynomial $p$ of degree $O\left( \frac{1}{\gamma^{4/5}} \log \frac{1}{\gamma}\right)$ in the $\ell_2$ norm, i.e. $\sqrt{\E[(f(x) - p(x))^2]} < \gamma$ \cite{FeldmanV15}. Specifically, $p$ is a truncation of the Fourier expansion of $f$, consisting only of terms with degree at most $O\left( \frac{1}{\gamma^{4/5}} \log \frac{1}{\gamma}\right)$. The polynomial $p$ can easily be extended to a linear function $h_f$ in $n^{O\left( \frac{1}{\gamma^{4/5}} \log \frac{1}{\gamma}\right)}$-dimensional space that closely approximates $f$ in the $\ell_2$ norm.

With this result in hand, we are in a similar position as we were in Section~\ref{sec:multiplicativeError}, although we now have a means of additively approximating a submodular function, rather than multiplicatively. However, the analysis from Section~\ref{sec:multiplicativeError} does not carry over directly. Recall that in that section, we knew that there existed a $p$-th power of a linear function that approximated the underlying function \emph{everywhere}. Now, due to the probabilistic nature of the $\ell_2$ norm, we can only say that \emph{in expectation}, for a set $S$ drawn at random from $\mathcal{D}$, $h_f(S)$ will be close to $f(S)$.

We address this subtlety in the design of our additive separation algorithm, Algorithm~\ref{alg:additive_error}, in the following way. First, we sample the sets $\sample_1$ and $\sample_2$ as before, and sort them according to the underlying submodular function $f$. Again, $\sample_1$ will serve as our landmarks; the size of $\sample_1$ is large enough so that we can be ensured that if $S$ and $S'$ are drawn uniformly at random, it is unlikely that $f(S)$ and $f(S')$ are at least $\beta$ apart additively, and yet there is no pair of landmarks $S_i$ and $S_j$ such that (1) $f(S_i)$ and $f(S_j)$ fall between $f(S)$ and $f(S')$ and (2) $f(S_i)$ and $f(S_j)$ are also separated by a $\beta$ additive factor.

Next, we solve for the suite of linear separators which allow the output comparator to make predictions. Recall that the linear function $h_f$ which approximates $f$ is in $n^{O\left( \frac{1}{\gamma^{4/5}} \log \frac{1}{\gamma}\right)}$-dimensional space, so rather than mapping each set $S$ to the characteristic vector $\chi(S)$ in order to solve for a linear separator, we must map it to $n^{O\left( \frac{1}{\gamma^{4/5}} \log \frac{1}{\gamma}\right)}$-dimensional space. This mapping is straightforward and is described later on.

The way in which we learn the linear separators in Algorithm~\ref{alg:additive_error} is the main difference between this algorithm and Algorithm~\ref{alg:active_learning}. This is because when designing Algorithm~\ref{alg:active_learning}, we knew that the approximating function approximated $f$ everywhere, whereas $h_f$ only approximates $f$ in the probabilistic $\ell_2$ norm. Therefore, it may be the case that two landmarks $S_i$ and $S_j$ are separated by a $\beta$ additive factor, yet for some $S$ such that $f(S) < f(S_i)$, it happens that $h_f(S) > h_f(S_i)$, for example. In other words, there may be noise in the training set $\sample_2$.

Therefore, when learning the suite of linear separators, we only save the indices of any landmark pair whose corresponding linear separator has low error rate over $\sample_2$. We ensure that $\sample_2$ is large enough so that it is unlikely that for any two landmarks $S_i$ and $S_j$, $f(S_i)$ and $f(S_j)$ are separated by a $\beta$ additive factor and yet the empirical risk minimizing linear separator that classifies $S$ as negative if $f(S) < f(S_i)$ and positive if $f(S) > f(S_j)$ has high error over $\sample_2$. Moreover, we ensure that $\sample_2$ is large enough so that it is unlikely that we learn a linear separator that has a much lower error rate over $\sample_2$ than over the entire distribution. We can conclude that the linear separators we keep track of will have low total error over the distribution.

Finally, as in Algorithm~\ref{alg:active_learning}, we keep track of only the ``minimal'' linear separators: we never keep track of a linear separator corresponding to a pair of landmarks $S_i$ and $S_j$ if we are also keeping track of a linear separator corresponding to a pair $S_{i'}$ and $S_{j'}$ such that $f(S_{i'})$ and $f(S_{j'})$ fall in between $f(S_i)$ and $f(S_j)$.

The output comparator predicts on input $S$ and $S'$ as in Algorithm~\ref{alg:active_learning}. It searches for a linear separator $(\vec{p}_{ij}, \theta_{ij})$ among the remaining minimal linear separators such that $S$ is classified as negative and $S'$ is classified as positive. If it finds one, it outputs 1 ($f(S) < f(S')$). Otherwise, it outputs 0 ($f(S') < f(S)$).

We now present our guarantees on the performance of Algorithm~\ref{alg:additive_error}.

\begin{algorithm}
\DontPrintSemicolon
\KwIn{Parameters $\epsilon, \delta, \beta \in (0,1)$ and $\ell$, a sample $\mathcal{S}$ from the uniform distribution over $2^{[n]}$ of size $\ell$, and a pairwise comparison oracle}
\KwOut{Function $g:2^{[n]}\times 2^{[n]}\rightarrow\{0,1\}$}
\begin{enumerate}
\item Set $\gamma = \beta\left(1 + \frac{2}{\epsilon} \log\frac{1}{\epsilon \delta}\sqrt{\frac{2}{\epsilon}}\right)^{-1}$, $k = \frac{25}{\gamma^{4/5}} \log \frac{\sqrt[3]{2}}{\gamma}$.
\item Remove $m=\frac{2}{\epsilon}\log(\frac{1}{\epsilon\delta})$ samples uniformly at random from $\mathcal{S}$. Label this set $\mathcal{S}_1$, and\\ label
$\mathcal{S}\setminus\mathcal{S}_1$ as $\mathcal{S}_2$.
\item Sort $\mathcal{S}_1$ such that $f(S_1)\leq \cdots\leq f(S_m)$ and sort $\sample_2$ into the sets $\sample_{ij} = \{S \ | \ S \in \sample_2 \text{ and } f(S) \leq f(S_i) \text{ or } f(S) \geq f(S_j)\}$ for all $S_i, S_j \in \sample_1$ such that $i < j$.
\item Let $S_1, \dots, S_{n^k}$ be an arbitrary ordering of the subsets of $[n]$ of size at most $k$, and for $S \subseteq [n]$, define $v(S)\in \R^{n^k}$ such that the $i^{th}$ component of $v(S)$ equals $\chi_{S_i}(S)$.
\item For each $S_i,S_j\in \mathcal{S}_1$, find $\theta_{ij}\in \mathbb{R}$ and $\vec{p}_{ij}\in \mathbb{R}^n$ that minimizes the number of sets\\ $S\in \sample_{ij}$, such that $f(S)<f(S_i)$ and $\vec{p}_{ij} \cdot v(S)>\theta_{ij}$ or $f(S_j)<f(S)$ and $\vec{p}_{ij}\cdot v(S)<\theta_{ij}$.\\ \label{step:add_train} If the fraction of such sets over $\sample_2$ is at most $\frac{\epsilon}{4m^2}$, put $(i,j)$ into $R$.
\item Remove all but the ``minimal'' pairs in $R$: if there exists
$(i,j),(i',j')\in R$ such that\\ $i\leq i'$ and $j\geq j'$, remove $(i,j)$ from $R$.
\item Define the function $g(S,S')$ as follows. Return 1 if $\exists (i,j)\in R$ such that $\vec{p}_{ij}\cdot v(S)<\theta_{ij}<\vec{p}_{ij}\cdot v(S')$.
Otherwise, return 0.
\end{enumerate}
\caption{Algorithm for learning submodular functions up to pairwise comparisons with an additive factor difference.}
\label{alg:additive_error}
\end{algorithm}

\begin{theorem}\label{thm:general_add}
Let $\mathcal{F}$ be the class of monotone submodular functions with range in $[0,1]$. For any $\beta \in (0,1)$, accuracy parameter $\epsilon$, and confidence parameter $\delta$, $\mathcal{F}$ is comparator learnable with separation $\beta$ given a sample of size $\tilde{O} \left(n^{O\left(\frac{1}{\gamma^{4/5}}\log \frac{1}{\gamma}\right)}/\epsilon^3\right),$ where $\gamma = \tilde{O}\left(\beta/\epsilon^{3/2}\right).$
\end{theorem}

We note that Algorithm \ref{alg:additive_error} is efficient given access to an ERM oracle for agnostically learning linear separators. Moreover, even in the model where the learning algorithm has access to real-valued function labels, for a monotone submodular function $f$ with range in $[0,1]$, the best known results for learning a function $h$ such that $||f-h||_2 \leq \epsilon$ require running time $2^{\tilde{O}(1/\epsilon^{4/5})}\cdot n^2$ and $2^{\tilde{O}(1/\epsilon^{4/5})}\log n$ random examples \cite{FeldmanV15}.

For ease of notation, we set $k = \frac{25}{\gamma^{4/5}} \log \frac{\sqrt[3]{2}}{\gamma}$ for the remainder of this section, where the constants come from the analysis in \cite{FeldmanV15}.

Let $f$ be a monotone submodular function with range in $[0,1]$. As we alluded to in the introduction of this section, we exploit existence of a polynomial $p$ of degree $O\left( \frac{1}{\gamma^{4/5}} \log \frac{1}{\gamma}\right)$ that closely approximates $f$ in the $\ell_2$ norm to show that there exists a vector $\vec{p}$ in $n^k$-dimensional space such that for $S,S' \sim \D$, with probability at least $1-\delta$, $f(S) + \beta < f(S')$ and $\vec{p} \cdot v(S) < \vec{p} \cdot v(S')$. Here, $v$ is a mapping from $n$-dimensional space to $n^k$-dimensional space, which we describe in the following analysis. Once we know that this vector $\vec{p}$ exists, we can attempt to solve for a set of the linear threshold functions $\vec{p}_{ij}, \theta_{ij}$ as in Algorithm 1, which will allow us to define the output predictor $g$. To this end, we now show that such a vector $\vec{p}$ in $n^k$-dimensional space does exist.

\begin{theorem}
Let $f:2^{[n]} \to [0,1]$ be a monotone submodular function, $\D$ be the uniform distribution over $2^{[n]}$, and $\gamma \in (0,1)$. There exists a vector $\vec{p}$ in $n^k$-dimensional space and a mapping $v$ from $2^{[n]}$ to $2^{\left[n^k\right]}$ such that $\sqrt{\E_{S \sim \D}[(f(S) - \vec{p} \cdot v(S))^2]} \leq \gamma.$
\end{theorem}

\begin{proof}
\cite{FeldmanV15} proved that for $\gamma \in (0,1)$, if we set $\kappa = \frac{1}{\gamma^{4/5}}\log \frac{1}{\gamma}$, then there exists $L \subseteq [n]$, $|L| \leq \frac{24}{\gamma^{4/5}} \log \frac{\sqrt[3]{2}}{\gamma},$ such that if $p(T) = \sum_{S:|S\setminus L| \leq \kappa} \hat{f}(S) \chi_S(T),$ then $||f-p||_2 = \sqrt{\E_{S \sim \D}[(f(S) - p(S))^2]} < \gamma$. Here, $\chi_S(T) = (-1)^{|T \cap S|}$ and $\hat{f}(S)$ is the Fourier coefficient of $f$ on $S$.

Unfortunately, we cannot find $L$ without knowing the value $f$ on our samples. However, we can simply extend the polynomial $p$ to include all summands from the Fourier expansion of $f$ up to sets of size $k \geq |L|+\kappa$ and thus obtain an even better approximation to $f$, a polynomial which we call $p_0$. Next, we can easily write $p_0$ as a linear mapping from $n^k$-dimensional space to $\R$ as follows. Let $S_1, \dots, S_\ell$ be an ordering of the sets $S \subseteq [n]$ such that $|S| \leq k$. Next, let \begin{align}\label{eq:pVec} \vec{p} &= \left( \hat{f}\left(S_1\right), \dots, \hat{f}\left(S_\ell\right)\right)\text{ and }  v(S) = \left(\chi_{S_1}(S) \dots, \chi_{S_\ell}(S)\right)\end{align} for all $S \subseteq [n]$. Then $p_0(S) = \vec{p} \cdot v(S)$.
\end{proof}
We are now in a similar situation as we were in Section~\ref{sec:multiplicativeError} when we were analyzing the case with a multiplicative approximation factor. In particular, we knew that so long as $\alpha(n) f(S) < f(S')$, then there had to exist a weight vector $\vec{w}$ such that $\vec{w} \cdot \chi(S) < \vec{w} \cdot \chi (S')$. Now, we know that there is some probability that $f(S) + \beta < f(S')$ and $\vec{p} \cdot v(S) < \vec{p} \cdot v(S')$. In the following lemmas, we derive a lower bound on that probability, which in turn allows us to provide strong guarantees on the performance of Algorithm~\ref{alg:additive_error}.

Lemma~\ref{lemma:add_tail} is an immediate consequence of Parseval's identity and Chebychev's inequality and the proof of Lemma~\ref{lem:add_concentration} can be found in Appendix~\ref{app:additiveError}.

\begin{lemma}\label{lemma:add_tail}
Given $\gamma,\xi \in (0,1)$, let $\vec{p}$ and $v$ be defined by Equation~\ref{eq:pVec}. Then \newline $\Pr_{S \sim \D}\left[|f(S) - \vec{p}\cdot v(S)| > \gamma \left(1+\sqrt{1/\xi}\right)\right] < \xi.$
\end{lemma}

\begin{lemma}\label{lem:add_concentration}
Given $\gamma,\xi \in (0,1)$, let $\vec{p}$ and $v$ be defined by Equation~\ref{eq:pVec}. Then \newline $\Pr_{S_1,S_2 \sim \D} \left[f(S_1) + 2\gamma\left(1+\sqrt{2/\xi}\right) < f(S_2) \text{ and }\vec{p}\cdot v(S_1) \leq \vec{p}\cdot v(S_2)\right] > 1-\xi.$
\end{lemma}

We are now ready to prove the correctness of Algorithm~\ref{alg:additive_error}.

\begin{proof}[Theorem~\ref{thm:general_add} proof sketch.]
In keeping with the outline of the proof of Theorem~\ref{thm:general}, we prove Claim~\ref{claim:buckets_additive}, a parallel to Claim~\ref{claim:buckets}, and Claim~\ref{claim:lin_sep_error_add}, a parallel to Claim~\ref{claim:lin_sep_error}. We then combine Claim~\ref{claim:buckets_additive} and Claim~\ref{claim:lin_sep_error_add} as we combined Claim~\ref{claim:buckets} and Claim~\ref{claim:lin_sep_error} to prove Theorem~\ref{thm:general_add}.

For the most part, Claim~\ref{claim:buckets_additive} follows from Claim~\ref{claim:buckets}. The proof of the latter is not specific to a multiplicative factor difference; it can easily be extended to an additive factor difference. In particular, it follows that if $\sample_1$, our set of ``landmarks'' which discretize the range of $f$, is sufficiently large, then on a random draw of $S$ and $S'$, it is unlikely that $f(S)$ and $f(S')$ are separated by an additive $\beta$ factor and yet there does not exist two landmarks $S_i$ and $S_j$ such that (1) $f(S_i)$ and $f(S_j)$ fall in between $f(S)$ and $f(S')$ and (2) $f(S_i)$ and $f(S_j)$ are separated by an additive factor of $\beta$.

However, this does not mean that a pair $(i,j)$ such that $f(S_i) + \beta < f(S_j)$ will necessarily be added to $R$. This again highlights the difference between this analysis and the analysis in Section~\ref{sec:multiplicativeError}. In that section, we were guaranteed that if $\alpha(n) f(S_i) < f(S_j)$, then there had to exist a linear threshold function that will label $S$ as negative if $f(S) < f(S_i)$ and positive if $f(S) > f(S_j)$. Now, we can only make a probabilistic argument about the likelihood that $f(S_i) + \beta < f(S_j)$ and $(i,j)$ is added to $R$. In particular, in Claim~\ref{claim:buckets_additive}, we show that if $|\sample_1|$ \emph{and} $|\sample_2|$ are sufficiently large, then on a random draw of $S$ and $S'$, it is unlikely that $f(S)$ and $f(S')$ are separated by an additive $\beta$ factor and yet there does not exist two landmarks $S_i$ and $S_j$ such that (1) $f(S_i)$ and $f(S_j)$ fall in between $f(S)$ and $f(S')$ and (2) the corresponding linear separator has small training error.

Next, we show in Claim~\ref{claim:lin_sep_error_add} that the probability that there exists a linear separator with much lower empirical error than true error is small. Since we only save linear separators that have small empirical error, this means that the true error will be small as well. In other words, for any linear separator corresponding to two landmarks $S_i$ and $S_j$ that we save, the probability that $f(S) > f(S_j)$ yet the linear separator classifies $S$ as negative is small, and the probability that $f(S) < f(S_i)$ yet the linear separator classifies $S$ as positive is small.

Finally, we rely on both Claim~\ref{claim:buckets_additive} and Claim~\ref{claim:lin_sep_error_add} to show that on a random draw of $S,S' \sim \mathcal{D}$, it is unlikely that $f(S)$ and $f(S')$ are separated by an additive $\beta$ factor and yet the algorithm predicts incorrectly. The formal way in which we combine Claim~\ref{claim:buckets_additive} and Claim~\ref{claim:lin_sep_error_add} to prove Theorem~\ref{thm:general_add} is similar to the proof of Theorem~\ref{thm:general}, and can be found in Appendix~\ref{app:additiveError}.
\end{proof}
Now, we provide the formal statement and proof of Claim~\ref{claim:buckets_additive}. The proof relies on the Chernoff bounding technique to show that for two landmarks $S_i$ and $S_j$, it is unlikely that that $f(S_i)$ and $f(S_j)$ are at least $\beta$ apart and yet the empirical risk minimizing linear separator has high empirical error over $\sample_2$.

\begin{claim}\label{claim:buckets_additive}
Sample sizes $|\sample_1| = O\left(\frac{1}{\epsilon}\log\frac{1}{\epsilon\delta}\right)$ and $|\sample_2| = O\left(\frac{m^2}{\epsilon} \log \frac{m^2}{\epsilon}\right)$ are
sufficient so that with probability $\geq 1-\delta$, \newline $\Pr_{S,S' \sim \D}\left[f(S')> f(S) + \beta \mbox{ and } \not\exists (i,j) \in R:
f(S') \geq f(S_j) > f(S_i) + \beta \geq f(S) + \beta\right] \leq \epsilon.
$
\end{claim}
\begin{proof}

First, we bound the probability that for a fixed $S_i, S_j \in \sample_1$, $f(S_i) + \beta < f(S_j)$, yet $(i,j) \not \in R$. Recall that the pair $(i,j)$ is added to $R$ in Step~\ref{step:add_train} of Algorithm~\ref{alg:additive_error} if, for the empirical risk minimizing threshold function $ERM_{ij}(\sample_2) = (\vec{p}_{ij}^*, \theta_{ij}^*)$, the fraction of sets $S \in \sample_2$ such that $f(S)<f(S_i)$ and $\vec{p}_{ij}^* \cdot v(S)>\theta_{ij}^*$ or $f(S_j)<f(S)$ and $\vec{p}_{ij}^*\cdot v(S)<\theta_{ij}^*$ is at most $\frac{\epsilon}{8m^2}$.

In other words, $(\vec{p}_{ij}^*, \theta_{ij}^*)$ is the linear threshold function which minimizes the loss function \[L_{(\vec{p}_{ij},\theta_{ij})}(S) = \begin{cases} 1 &\text{if } f(S) \leq f(S_i) \mbox{ and } \vec{p}_{ij} \cdot v(S) > \theta_{ij}\\
&\text{or } f(S) \geq f(S_j) \mbox{ and } \vec{p}_{ij} \cdot v(S) < \theta_{ij}\\
0 &\text{otherwise} \end{cases}.\]

Let $|\sample_2| = m'$. For a given $(\vec{p}_{ij}, \theta_{ij})$, let \[L_{(\vec{p}_{ij},\theta_{ij})}(\sample_2) = \frac{1}{m'} \sum_{S\in \sample_2} L_{(\vec{p}_{ij},\theta_{ij})}(S)\] be the empirical loss of $(\vec{p}_{ij},\theta_{ij})$ over $\sample_2$. We claim that a sample size of $m' = O\left(\frac{m^2}{\epsilon} \log \frac{m^2}{\epsilon}\right)$ is sufficient to ensure that $\Pr_{\mathcal{S}_2 \sim \mathcal{D}^{m'}}[L_{(\vec{p}_{ij}^*,\theta_{ij}^*)}(\sample_2) > \frac{\epsilon}{8m^2}] < \frac{\epsilon}{2m^2}.$ We prove this using Chernoff's bounding technique.

First, notice that \begin{align*}\E_{S \sim \mathcal{D}} [L_{(\vec{p}_{ij},\theta_{ij})}(S)] = &\Pr[f(S) \leq f(S_i) \text{ and } \vec{p}_{ij} \cdot v(S) > \theta_{ij}]\\ + &\Pr[f(S_j) \leq f(S) \text{ and } \vec{p}_{ij} \cdot v(S) < \theta_{ij}].\end{align*} We will see that it is enough to find an upper bound on $\E_{S \sim \mathcal{D}}[L_{(\vec{p},\tilde{\theta}_{ij})}(S)]$, where $\tilde{\theta}_{ij} = \frac{f(S_i) + f(S_j)}{2}$ and $\vec{p}$ is defined by Equation~\ref{eq:pVec}.

We begin by finding an upper bound on $\Pr[f(S_j) \leq f(S) \text{ and } \vec{p} \cdot v(S) < \tilde{\theta}_{ij}]$. Notice that $\tilde{\theta}_{ij} = \frac{f(S_i) + f(S_j)}{2} < f(S_j) - \frac{\beta}{2}$ since $f(S_j) > f(S_i) + \beta$. Therefore, \[\Pr[f(S_j) \leq f(S) \text{ and } \vec{p} \cdot v(S) < \tilde{\theta}_{ij}] < \Pr \left[f(S_j) \leq f(S) \text{ and } \vec{p} \cdot v(S) \leq f(S_j) - \frac{\beta}{2}\right].\] However, so long as $|f(S) - \vec{p}\cdot v(S)| < \frac{\beta}{2}$, we know that $f(S_j) - \frac{\beta}{2} \leq f(S) - \frac{\beta}{2} < \vec{p}\cdot v(S)$. Therefore, the only way that $\vec{p}\cdot v(S) \leq f(S_j) - \frac{\beta}{2}$ is if $|f(S) - \vec{p}\cdot v(S)| \geq \frac{\beta}{2}$, which we know from Lemma~\ref{lemma:add_tail} happens with probability at most $\frac{\epsilon}{8m^2}$ since \begin{align*}\frac{\beta}{2} = \frac{\gamma}{2}\left(1 + \frac{2}{\epsilon} \log\frac{1}{\epsilon \delta}\sqrt{\frac{2}{\epsilon}}\right)= \frac{\gamma}{2} \left( 1 + \sqrt{\frac{8m^2}{\epsilon}}\right).\end{align*} To apply Lemma~\ref{lemma:add_tail}, we simply use $\frac{\gamma}{2}$ instead of $\gamma$ and $\frac{\epsilon}{8m^2}$ instead of $\delta.$ Therefore, \[\Pr[f(S_j) \leq f(S) \text{ and } \vec{p} \cdot v(S) < \tilde{\theta}_{ij}] < \frac{\epsilon}{8m^2}.\] By a symmetric argument, \[\Pr[f(S) \leq f(S_i) \text{ and } \vec{p} \cdot v(S) > \tilde{\theta}_{ij}]< \frac{\epsilon}{8m^2}\] as well. This means that $\E_{S \sim \mathcal{D}}[L_{(\vec{p},\tilde{\theta}_{ij})}(S)] < \frac{\epsilon}{4m^2}.$

Using Chernoff's bounding technique, we have that \[\Pr_{\sample_2 \sim \mathcal{D}^{m'}}\left[L_{(\vec{p},\tilde{\theta}_{ij})}(\sample_2) > \frac{\epsilon}{4m^2}\right] < \frac{\epsilon}{2m^2}\] for $m' = O\left(\frac{m^2}{\epsilon} \log \frac{m^2}{\epsilon}\right).$ Since it will always be the case that $L_{ERM_{ij}(\sample_2)}(\sample_2) < L_{(\vec{p},\tilde{\theta}_{ij})}(\sample_2)$, we have that  \[\Pr_{\sample_2 \sim \mathcal{D}^{m'}}\left[L_{ERM_{ij}(\sample_2)}(\sample_2) > \frac{\epsilon}{4m^2}\right] < \frac{\epsilon}{2m^2}\] as well. By a union bound over all $m^2$ $(i,j)$ pairs in $\sample_1$, we have that for $S,S' \sim \mathcal{D}$, the probability that $f(S') > f(S) + \beta$ and there exists $S_i,S_j \in \sample_1$ such that $f(S') \geq f(S_j) > f(S_i) + \beta \geq f(S) + \beta$ yet $(i,j) \not \in R$ is at most $\frac{\epsilon}{2}$.

Finally, by the same argument as in the proof of Claim~\ref{claim:buckets_additive}, a sample $\sample_1$ of size $|\sample_1| = O\left(\frac{1}{\epsilon}\log\frac{1}{\epsilon\delta}\right)$ is sufficient the guarantee that with probability at least $1-\delta$, the probability that $f(S') > f(S) + \beta$ and there does not exist $S_i,S_j \in \sample_1$ such that $f(S') \geq f(S_j) > f(S_i) + \beta \geq f(S) + \beta$ is at most $\epsilon$. Putting these two arguments together, we have that with probability at least $1-\delta$, \[\Pr_{S,S' \sim \D}\left[f(S')> f(S) + \beta \mbox{ and } \not\exists (i,j) \in R:
f(S') \geq f(S_j) > f(S_i) + \beta \geq f(S) + \beta\right] \leq \epsilon.\]
\end{proof}

Finally, we state Claim~\ref{claim:lin_sep_error_add}. The proof takes advantage of agnostic learning VC dimension bounds to prove that the size of $\sample_2$ is sufficiently large to ensure that the true error of the learned linear separators is close to the empirical error. The full proof can be found in Appendix~\ref{app:additiveError}.

\begin{claim}\label{claim:lin_sep_error_add}

A sample size $|\sample_2| = O\left( \frac{1}{\epsilon^3}\log\frac{1}{\epsilon\delta}\left[n^{O\left(\frac{1}{\gamma^{4/5}}\log \frac{1}{\gamma}\right)} \log \left(\frac{1}{\epsilon^2}\log\frac{1}{\epsilon\delta}\right) + \log \frac{1}{\delta}\right]\right)$ is sufficient so that with probability at least $1-\delta$, \[\Pr_{S \sim \D}
\left[\begin{array}{r}\exists (i,j)\in R \text{ such that } f(S) \leq f(S_i)
\mbox{ yet } \vec{p}_{ij} \cdot \chi(S) > \theta_{ij}\\
\text{or }f(S) \geq f(S_j)
\mbox{ yet } \vec{p}_{ij} \cdot \chi(S) < \theta_{ij} \end{array}\right] < \frac{\epsilon}{2}.\]
\end{claim}

We can obtain similar results for the class of XOS functions, as follows.

\begin{cor}
Let $\mathcal{F}$ be the class of XOS functions with range in $[0,1]$. For any $\beta \in (0,1)$, accuracy parameter $\epsilon$, and confidence parameter $\delta$, $\mathcal{F}$ is comparator learnable with additive separation $\beta$ given a sample of size $\tilde{O}\left( n^{O\left(1/\gamma\right)}/\epsilon^3\right)$, where $\gamma = \tilde{O}\left( \beta / \epsilon^{3/2}\right)$.
\end{cor}
\begin{proof}
This follows from the same line of reason as in the proof of Theorem~\ref{thm:general_add} and the fact that for any XOS function $f: 2^{[n]} \to [0,1]$, there is a polynomial $p$ of degree $O(1/\gamma)$ such that $||f - p||_2 \leq \gamma$ \cite{FeldmanV15}. In particular, $p(T) = \sum_{S : |S| \leq \frac{\sqrt{5}}{(2\gamma)}} \hat{f}(S) \chi_S(T).$
\end{proof}

\section{Application to Other Combinatorial Functions}\label{sec:otherCombinatorial}

We can extend Algorithm~\ref{alg:active_learning} to learn over many other classes of combinatorial functions up to pairwise comparisons, including valuation functions with limited nonlinear interactions \cite{vainsencher2011bundle}, Fourier sparse set functions \cite{stobbe2012learning}, and coverage functions \cite{Sketching,feldman2014learning}. We summarize the function classes we investigate in the following sections, providing a brief motivation and a description of our guarantees.
\subsection{Valuation Functions with Limited Nonlinear Interactions} \label{sec:nonlinear}

A valuation function $f$ is simply a set function such that $f(\emptyset) = 0$. We show that if, intuitively speaking, the underlying valuation function expresses nonlinear interactions between sets of size at most $k$ (i.e. it is a function with \emph{$k$-limited nonlinear interactions}), then Algorithm~\ref{alg:active_learning} learns $f$ up to comparisons using a sample of size $\tilde{O}(n^k / \epsilon ^3)$. Notably, we do not require, on input $S$ and $S'$, that $f(S)$ and $f(S')$ be sufficiently far apart in order to guarantee that the learned comparator will predict correctly with high probability. This is in contrast to the previous results, where in order the guarantee that the learned comparator will predict correctly, we required that $c f(S) < f(S')$ for $c$ sufficiently large.

To define what we mean by limited nonlinear interactions, we use the notion of an \emph{interaction function} \cite{vainsencher2011bundle}. Let $f: 2^{[n]} \to \R$ be a valuation function and let $g: 2^{[n]} \setminus \emptyset: \to \R$ be defined such that for all $S \subseteq [n]$, \[f(S) = \sum_{T \subseteq [n] : S \cap T \not= \emptyset} g(T).\] The function $g$ is called the \emph{interaction function of $f$}. Vainsencher et al.\ proved that every valuation function $f$ has a unique interaction function $g$ such that $g(\emptyset) = 0$ \cite{vainsencher2011bundle}. We then say that $f$ has degree $k$ if for all $T \in 2^{[n]}$ such that $|T| > k$, $g(T) = 0$ and we define $F_k$ to be the set of valuation functions $f$ of degree at most $k$. Intuitively, $F_k$ contains all valuation functions that express nonlinear interactions between subsets of size at most $k$.

We note that $F_k$ contains many natural valuation functions. A valuation function $f$ likely falls in $F_k$ when the $n$ objects in the ground set either form instances of $j$-wise complements or $j$-wise substitutes, where $j$ is at most $k$. For example, stamps are typically produced as members of a small collection under a unified theme, such as recently released movies or commemorations of a country's leaders or monuments. A stamp collector will likely value a set $S$ of stamps from the same collection more than she will value any stamp in $S$ on its own, and thus her valuation function is supermodular. Moreover, if $k$ is an upper bound on the size of any collection she has her eye on, then it possible to show that her valuation function will fall in $F_k$.

Vainsencher et al.\ suggested sensor placement as another natural application domain. Each problem instance consists of a set of $n$ possible points where sensors may be placed, and the valuation function is determined by the amount of area covered by a particular selection of sensor placements. If at most $k$ sensors cover any point in the domain, then it is easy to see that the valuation function is submodular and falls in $F_k$.

With these motivations in mind, we now summarize our guarantees in Theorem~\ref{thm:general_valuation} regarding the performance of Algorithm~\ref{alg:active_learning} when we know that the underlying valuation function $f$ is a member of $F_k$. The proof can be found in Section~\ref{app:other_combinatorial}.

\begin{theorem} \label{thm:general_valuation}
For all $1\leq k\leq n$,
the class of functions in $F_k$ is comparator-learnable with sample complexity and runtime polynomial in $n^k$, $1/\epsilon$, and $1/\delta$, 
and with no separation factor, using Algorithm \ref{alg:active_learning}.
\end{theorem}

\subsection{Fourier Sparse Set Functions} \label{sec:fourier}

We can extend Algorithm~\ref{alg:active_learning} to general combinatorial functions with Fourier support contained in a set $\mathcal{P} \subseteq 2^{[n]}$. Moreover, we achieve even better sample complexity if we are guaranteed that the size of the Fourier support is bounded by a constant $k \leq |\mathcal{P}|$. Again, we do not need to require that $f(S)$ and $f(S')$ be separated by a sufficiently large multiplicative factor in order to guarantee that the learned comparator predicts correctly with high probability.

An important example of a function with sparse Fourier support is the cut function of a graph $G = (V,E)$, equipped with a weight function $w:E \to \R$. This function is defined as $f_G(A) = \sum_{s \in A, t \in V\setminus A} w(s,t).$ Stobbe and Krause show that the Fourier support of $f$ is contained in $\mathcal{P} = \{ S \ | \ |S| = 2\} \cup \emptyset$ \cite{stobbe2012learning}. 
We show that we can learn $f$ up to comparisons using $\tilde O\left(\frac{|V|^2}{\epsilon^3}\right)$ examples in time polynomial in $|V|$, $1/\epsilon$, and $1/\delta$.

Building on this, the binary function $f_G(A,B)= \sum_{s \in A, t \in B} w(s,t)$ can be seen as a function over $\{0,1\}^{2n}$ with Fourier support contained in $\mathcal{P} = \{ S \ | \ |S| \leq 2\}$. Therefore, we can learn $f$ up to comparisons using $\tilde O\left(\frac{|V|^2}{\epsilon^3}\right)$ examples in time polynomial in $|V|$, $1/\epsilon$, and $1/\delta$. In analysis of social networks, one might wish to learn the function $f_G(A,B)$ up to comparisons as a means to order the influence of individuals and groups in that network.

We now present our main result for learning Fourier sparse set functions, the proof of which can be found in Appendix~\ref{app:other_combinatorial}.

\begin{theorem}\label{thm:Fourier_sparse}
For all $k \geq 1$,
the class of functions with at most $k$ nonzero Fourier coefficients and support contained in $\mathcal{P} \subseteq 2^{[n]}$ is comparator-learnable with sample complexity polynomial in $k$, $1/\epsilon$, and $1/\delta$, running time polynomial in $|\mathcal{P}|^{\Theta(k)},$ $1/\epsilon$, and $1/\delta$, and with no separation. Alternatively, this class is also comparator-learnable using running time and sample complexity polynomial in $|\mathcal{P}|,$ $1/\epsilon$, and $1/\delta$ and with no separation.
\end{theorem}

\subsection{Coverage Functions}
Coverage functions form a subset of the class of submodular functions, and have applications in combinatorial optimization, machine learning, and algorithmic game theory. A coverage function $f$ is defined on $[n]$, and each element in $[n]$ corresponds to a subset of a universe $U$, whose elements have nonnegative weights. The value of $f$ on $S \subseteq [n]$ is the weight of the union of the corresponding subsets in $U$.

We combine structural results specific to coverage functions from \cite{Sketching} and \cite{feldman2014learning} to prove that coverage functions are comparator-learnable with multiplicative separation $(1+\epsilon)$, using $\tilde{O}(n^3/\epsilon^5)$ queries, given access to an ERM oracle for learning linear separators. In particular, we prove the following theorem, the proof of which can be found in Appendix~\ref{app:other_combinatorial}.

\begin{theorem}\label{thm:coverage}
The class of coverage functions is comparator-learnable with multiplicative separation $(1+\epsilon)$ and sample complexity polynomial in $n, 1/\epsilon,$ and $1/\delta$.
\end{theorem}

\section{Conclusion}
In this paper, we explore the learning model where the goal is to learn an underlying combinatorial function up to pairwise comparisons, from pairwise comparisons. We present several general algorithms that can be used to learn over a variety of combinatorial function classes, including those consisting of submodular, XOS, subadditive, coverage, and Fourier sparse functions. We also prove nearly matching lower bounds for submodular and XOS functions, and for submodular functions with bounded curvature.

In particular, we show that if the functions in a class $\mathcal{F}$ can be approximated by the $p$-th root of a linear function to within a multiplicative factor of $\alpha(n)$, then we can learn a comparison function that for most pairs $S,S' \sim \mathcal{D}$, either predicts correctly which of $f(S)$ or $f(S')$ is larger, or $f(S)$ and $f(S')$ are separated by less than a multiplicative $\alpha(n)$ factor. We extend this algorithm to account for an additive separation factor, rather than a multiplicative separation factor, by taking advantage of key structural properties of the Fourier spectrum of the functions we consider. In this case, we require that the underlying distribution be uniform and that the underlying function be XOS or monotone submodular with range in $[0,1]$. Finally, we show that it is possible to learn over some combinatorial function classes, such as the class of Fourier sparse functions, with no separation factor. In this way, the power and adaptability of our general algorithmic framework is exemplified by our results over a hierarchy of function classes, with significantly stronger separation factor guarantees the more structure a class exhibits.

Determining the exact approximation factor for comparator-learning submodular
functions is an open question, as there is a gap between
the $\tilde{O}(n^{1/2})$
upper bound and the $\tilde{\Omega}(n^{1/3})$ lower bound.
Another open question is determining whether the sample complexity for the
additive error results in
Section \ref{sec:additiveError} can be improved.
We note, both of these questions are unresolved even in the setting
where the sample consists of function values and the goal is to learn an
approximate function.
Another interesting question is to find nontrivial generalizations
of the pairwise comparison model and show corresponding results.
For instance, the distribution is over $k$-tuples and the
top $k'$ sets in the tuple are ranked.

\section*{Acknowledgments}

We thank Peter Bartlett for insightful initial discussions which led to the development of this research project.

This work was supported in part by NSF grants  CCF-1451177,
CCF-1422910, a Sloan Research Fellowship, a Microsoft Research Faculty
Fellowship, a Google Research Award,
and a National Defense Science \& Engineering Graduate (NDSEG) fellowship. 

\bibliography{bibliography}
\bibliographystyle{plain}

\clearpage
\appendix
\section{Learning $k$-submodular functions} \label{ksubmodular}

In this section, we consider $k$-submodular functions under membership queries.
In this model, there is no restriction on the sets we can query, but we must
learn the function up to comparisons with probability 1 (perhaps with a separation
condition).
We start with a simple result about boolean submodular functions, i.e., submodular
functions with range in $\{0,1\}$.

\begin{theorem}
Given a boolean submodular function $f$, with $\frac{3n}{2}$ comparison 
oracle calls, we can learn $f$ exactly.
\end{theorem}

\begin{proof}
We use the well-known structural result that any boolean submodular function can be
represented exactly by a disjunction.
Given the elements of the ground set $x_1,\dots,x_n$,
then $f=\vee_{x\in S} x$, for some unknown subset $S$ of the ground set.
Without loss of generality, we assume that $f(\emptyset)=0$.
Then for each $x_i$, we call $g(x_i,\emptyset)$. If the output is 0, then $f(x_i)=0$,
otherwise $f(x_i)=1$.
It is clear that $S=\{x\mid f(x)=1\}$. Therefore, we now have an exact expression for
$f$.
\end{proof}

Now we move to $k$-submodular functions, which are submodular functions with range in
$[k]$. First, we need the following definition.

\begin{definition}
A pseudo-boolean $k$-DNF is a DNF $f(x_1,\dots,x_n)=\max_{t=1}^s (a_t \vee_{i\in A_t} x_i)$,
where $a_t$ are constants, $A_t\subseteq [n]$, and $|A_t|\leq k$ for $t\in [s]$.
\end{definition}

\begin{theorem}\label{thm:n^k}
Given a $k$-submodular function $f$, with $O(n^k)$ comparison oracle calls,
we can learn $f$ up to comparisons.
\end{theorem}

\begin{proof}
We use the result from Raskhodnikova and Yaroslavtsev, which states that any $k$-submodular
function can be represented as a pseudo-boolean $2k$-DNF with constants $a_t \in [k]$
\cite{raskhodnikova2013learning}.
So let $f=\max_{t=1}^s (a_t \vee_{i\in A_t} x_i)$ for constants $a_t$ and sets $A_t$.
Now, for all $A_t$, define a new variable 
$y_t=\wedge_{i\in A_t} x_i$. There are $\leq{2n\choose k}=s$ such variables.
Then $f(x_1,\dots x_n)=\max_{t=1}^s (a_t y_t)$. In other words, $f(S)$ takes the value of
the maximum $a_t$ such that $\{x_i\mid i\in A_t\}\subseteq S$.

Now we can learn $f$ up to comparisons by sorting all $y_i$, which can be done with
$\leq s \log{s}$ comparison oracle calls. Since $f$ can only take on values in
$\{0,\dots, k\}$, this will give us $k'\leq k+1$ buckets $B_1,\dots B_{k'}$,
where for all $i,j\in [s]$, $f(y_i)\leq f(y_j)$ iff $i'\leq j'$, 
where $y_i\in B_{i'}$, $y_j\in B_{j'}$.

On a comparison query $(A,B)$, we output the set which has a subset in the bucket with the
largest index. I.e., set $x_A=\text{argmax}_i (\exists S\subseteq A\mid S\in B_i)$
and $x_B=\text{argmax}_i (\exists S\subseteq B\mid S\in B_i)$, and then $f(A)\leq f(B)$
iff $x_A\leq x_B$.
\end{proof}

\begin{theorem}\label{thm:n^k/alpha}
Given a monotone $k$-submodular function $f$ and a parameter $\alpha$ which divides $2k$, there is an algorithm that learns $f$ well enough to predict accurately whether or not $f(S) \geq f(S')$ whenever $f(S) > \alpha f(S')$ or $f(S') > \alpha f(S)$ using $O\left(\frac{k}{\alpha}n^{k/\alpha}\log n\right)$ comparison oracle calls. The algorithm's running time is $O\left(\frac{k}{\alpha}n^{k/\alpha}\log n\right)$, and each prediction takes $O\left(\frac{k}{\alpha}n^{k/\alpha}\right)$ steps.
\end{theorem}

\begin{proof}
First, let $S \subseteq [n]$ be such that $|S|>2k/\alpha$. We claim that there exists some $Q \subseteq S$ such that $f(Q) \leq f(S) \leq \alpha f(Q)$. To see this, define the set \[S_0 = \begin{cases} S &\text{ if } |S| \leq 2k\\
S' &\text{ otherwise}\end{cases},\] where $S' \subseteq S$ is a set of size $2k$ such that $f(S) = f(S')$. We know that such a set exists from the argument in the proof of Theorem \ref{thm:n^k}. Now, let $\mathcal{P} = \{S_1, \dots, S_{\alpha}\}$ be a partition of $S_0$, where $|S_i| \leq |S_0|/\alpha \leq 2k/\alpha$. Then we have that \begin{align*}f(S) = f(S_0) &\leq \sum_{i = 1}^{\alpha} f(S_i) &(\text{by subadditivity})\\ &\leq \alpha \max_{i \in [\alpha]} \{f(S_i)\}.\end{align*}

Let $S_0^* = \max_{i \in [\alpha]} \{f(S_i)\}$. Since $S_0^* \subset S$, we know by monotonicity that $f(S_0^*) \leq f(S) \leq \alpha f(S_0^*)$.

Now, suppose that using comparison queries, we sort all $S \subseteq [n]$ of size at most $2k/\alpha$ into buckets $B_1, \dots, B_{k'}$, where $k' \leq k+1$. Such that for all $i<j \leq k'$, if $S,S' \in B_i$, then $f(S) = f(S')$, if $Q \in B_i$ and $Q' \in B_j$, then $f(Q) < f(Q')$. This takes $O\left(\frac{k}{\alpha}n^{k/\alpha}\log n\right)$ queries.

We are now ready to describe how to predict on two sets $(S,S')$. First, we find \[i = \underset{i \in [k']}{\text{argmax}} \{B_i \ | \ \exists Q \in B_i, Q \subseteq S\}, j = \underset{j \in [k']}{\text{argmax}} \{B_j \ | \ \exists Q \in B_j, Q \subseteq S'\}.\] If $i \leq j$, then we output $f(S) \leq f(S')$ and otherwise, we output $f(S) > f(S')$.

We claim that, without loss of generality, $f(S) > \alpha f(S')$, then we predict correctly. To see this, let $Q \subseteq S$ be a subset in $\max_{i \in [k']}\{B_i \ | \ \exists Q \in B_i, Q \subseteq S\}$ and $Q' \subseteq S'$ be a subset in be a subset in $\max_{i \in [k']}\{B_i \ | \ \exists Q \in B_i, Q \subseteq S'\}$. Then we know that \[\alpha f(Q') \leq \alpha f(S') < f(S) \leq \alpha f(Q),\] so $f(Q') < f(Q)$. Therefore, we output the correct ordering.

For any $S$, the running time to find $\max_{i \in [k']}\{B_i \ | \ \exists Q \in B_i, Q \subseteq S'\}$ is $O\left(\frac{k}{\alpha}n^{k/\alpha}\right)$, since there are $O(n^{k/\alpha})$ sets in the buckets, and it takes $O\left(\frac{k}{\alpha}\right)$ steps to check if any subset $Q$ of size at most $2k/\alpha$ is a subset of $S$.
\end{proof}

\section{Proofs from Section \ref{sec:multiplicativeError}} \label{app:multiplicativeError}

\begin{proof}[Proof of Claim \ref{claim:lin_sep_error}]
First we will show the probability of the bad event happening 
for each $(i,j)\in R$ is low, and then we union bound over all possible pairs
$(i,j)$. Formally, let
$$P_1 = \Pr_{S \sim \D}[f(S) \leq f(S_i)\mbox{ and } (i,j)\in R \mbox{ yet } \vec{w}_{ij} \cdot \chi(S) > \theta_{ij}]$$
and
$$P_2 = \Pr_{S' \sim \D}[f(S') \geq f(S_j) \mbox{ and } (i,j)\in R \mbox{ yet } \vec{w}_{ij} \cdot \chi(S') < \theta_{ij}].$$ 
We show that with probability at least $1-\delta$, 
$P_1+P_2 \leq \frac{\epsilon}{m^2}.$

First we bound the VC dimension of each class of loss functions 
$\mathcal{L}_{ij} = \{L_{(\vec{w}_{ij},\theta_{ij})} \ | \ \vec{w}_{ij} \in \R^n, \theta_{ij} \in \R\},$ where 
\[L_{(\vec{w}_{ij},\theta_{ij})}(S) = 
\begin{cases} 1 &\text{if } f(S) \leq f(S_i) \mbox{ and } \vec{w}_{ij} \cdot \chi(S) > \theta_{ij}\\
&\text{or } f(S) \geq f(S_j) \mbox{ and } \vec{w}_{ij} \cdot \chi(S) < \theta_{ij}\\
0 &\text{otherwise} \end{cases}.\]

Clearly, if $\sample = \{S_1, \dots, S_t\}$ can be shattered, then it cannot contain any set $S$ such that $f(S_i) < f(S) < f(S_j)$ because for such a set $S$, $L_{(\vec{w}_{ij},\theta_{ij})}(S) = 0$ for all choices of $(\vec{w}_{ij},\theta_{ij})$. Therefore, the problem reduces to finding the VC dimension of the class of linear separators over $\R^n$ in the following way.

First, suppose that $\sample$ can be labeled in every possible way by the class of linear separators over $\R^n$. We claim that for any $A \subseteq [t]$, there exists a linear separator $(\vec{w}_{ij},\theta_{ij})$ such that $L_{(\vec{w}_{ij},\theta_{ij})}(S_\ell) = 1$ if $\ell \in A$ and $L_{(\vec{w}_{ij},\theta_{ij})}(S_\ell) = 0$ if $\ell \not\in A$. To this end, let \begin{align*}A_{\leq} &= \{\ell \ | \ \ell \in A, f(S_\ell) \leq f(S_i)\},\\
A_{\geq} &= \{\ell \ | \ \ell \in A, f(S_\ell) \geq f(S_j)\},\\
B_{\leq} &= \{\ell \ | \ \ell \not\in A, f(S_\ell) \leq f(S_i)\}, \text{ and}\\
B_{\geq} &= \{\ell \ | \ \ell \not\in A, f(S_\ell) \geq f(S_j)\}.\end{align*} Then $(\vec{w}_{ij}, \theta_{ij})$ is the linear separator that labels $A_{\leq} \cup B_{\geq}$ as positive and $A_{\geq} \cup B_{\leq}$ as negative. Such a linear separator must exist by assumption.

Now suppose that $\sample$ can be labeled in every possible way by functions in $\mathcal{L}_{ij}$. We claim that for any $A \subseteq [t]$, there exists a linear separator $(\vec{w},\theta)$ such that $\vec{w}\cdot \chi(S_\ell) \geq \theta$ if $\ell \in A$ and $\vec{w}\cdot \chi(S_\ell) < \theta$ if $\ell \not\in A$. Indeed, $(\vec{w},\theta)$ is the linear separator such that $L_{(\vec{w},\theta)}(S_\ell) = 1$ if $\ell \in A_{\leq} \cup B_{\geq}$ and $L_{(\vec{w},\theta)}(S_\ell) = 0$ if $\ell \in A_{\geq} \cup B_{\leq}$. We know that $(\vec{w},\theta)$ by assumption.

Therefore, the VC dimension of $\mathcal{L}_{ij}$ is the same as the VC dimension of the class of linear separators over $\R^n$, which is $n+1$.

Now, let $h^*$ be the empirical risk minimizer of $\mathcal{L}_{ij}$ over $\sample_2$. We know that so long as $\sqrt{n} f(S_i) \leq f(S_j)$, $h^*$ will have zero loss over the $\sample_2$. Therefore, by standard VC dimension bounds, if $|\sample_2| = O(\frac{m^2}{\epsilon}[n \log \frac{1}{\epsilon} + \log\frac{1}{\delta m^2}])$, then with probability at least $1-\delta$, the error of $h^*$ over $\D$ is at most $\frac{\epsilon}{m^2}$. 
Now we union bound over all $m^2$ pairs $(i,j)$ on which 
Algorithm \ref{alg:active_learning} attempts to solve for a linear threshold 
$(\vec{w}_{ij}, \theta_{ij})$, to achieve an overall error of $\epsilon$.
\end{proof}

\begin{proof}[Proof of Corollary \ref{cor:structure_classes}]
The statements follow by applying Theorem \ref{thm:general} with the following structural results.
\begin{enumerate}
\item The class of submodular functions is $\sqrt{n}$-approximable because for all submodular functions $f$, there exists $w\in\mathcal{R}^n$ such that
$f(S)\leq \sqrt{w\cdot\chi(S)}\leq \sqrt{n}f(S)$ \cite{goemans2009approximating}.
\item The class of XOS functions is $O(\sqrt{n})$-approximable because for all XOS functions $f$, there exists $w\in\mathcal{R}^n$ such that
$f(S)\leq \sqrt{w\cdot\chi(S)}\leq \alpha(n) f(S)$, where $\alpha(n) = O(\sqrt{n})$ \cite{balcan2012learning}.
\item The class of subadditive functions is $\sqrt{n}\log n$-approximable because for all subadditive functions $f$, there exists a submodular function $g$ such that
for all $S \subseteq [n]$, $f(S)\leq g(S)\leq f(S)\log{n}$ \cite{Balcan11submodular}. From this, we can use the approximation guarantee for submodular functions from item 1 of this corollary to obtain the result.
\item Recall the curvature of a submodular function $f$ is defined as
\begin{equation*}
\kappa_f=1-\min_{j\in [n]}\frac{f([n])-f([n]\setminus\{j\})}{f(j)}.
\end{equation*}
Intuitively, the curvature $f$ is the extent to which the function deviates from a modular function.
From \cite{iyer2013curvature}, we have the following bound,
\begin{equation*}
\forall S\subseteq [n],~(1-\kappa_f)\sum_{j\in[n]}f(j)\leq f(S)\leq \sum_{j\in[n]}f(j).
\end{equation*}
Therefore, the class of functions with curvature at most $\kappa$ is $\min\left\{\sqrt{n},\frac{1}{1-\kappa}\right\}$-approximable.
\item \cite{balcan2012learning} proved that for any XOS function $f$ with $R$ SUM trees, there exists
a function $g(S)=w\cdot \chi_M(S)$ such that $f(S)\leq g(S)\leq R^{\xi}\cdot g(S)$,
where $\chi_M$ denotes the indicator function for all subsets of size at most $\frac{1}{\xi}$ over $[n]$. Specifically, $\chi_M(S)_{i_1,i_2,\dots,i_L} = 1$ if $\{i_1, i_2, \dots i_L\} \subseteq S$ and $\chi_M(S)_{i_1,i_2,\dots,i_L} = 0$ otherwise. We feed the sample with the new features into Algorithm~\ref{alg:active_learning} in order to learn a comparator with separation $R^\xi$ for this class of functions. Now that the feature space consists of $n^{1/\xi}$ features, the sample complexity and running time are polynomial in $n^{1/\xi}$, rather than $n$.
\end{enumerate}
\end{proof}

\begin{proof}[Proof of Theorem~\ref{thm:XOS_lower}]
First, we use Theorem 1 from \cite{balcan2012learning}, which guarantees that there exists a family of subsets $\mathcal{A} = \left\{A_1, \dots, A_k\right\} \subset 2^{[n]}$ such that for any $\mathcal{B} \subseteq \mathcal{A}$, there exists an XOS function $f_{\mathcal{B}}$ such that $f_{\mathcal{B}}(A_i) = \Omega(\sqrt{n})$ if $A_i \in \mathcal{B}$ whereas $f_{\mathcal{B}}(A_i) = O(\log n)$ if $A_i \not \in \mathcal{B}$. Moreover, $k = n^{\frac{1}{3}\log \log n}$.

Just as in the proof of Theorem~\ref{thm:sub_lower}, this implies hardness for comparator-learning over the uniform distribution on $\mathcal{A}$ from any polynomial-sized samples. After all, for an arbitrary algorithm $\mathcal{ALG}$, suppose that $\mathcal{ALG}$ has access to pairwise comparisons over a sample of sets $\sample$, where $|\sample| = \ell \leq n^c$ for some constant $c$. Then in expectation, half of the non-training set samples have values in $O(\log n)$ and half have values in $\Omega(\sqrt{n})$. This follows by the same argument as in the proof of Theorem~\ref{thm:sub_lower}. Moreover, the measure of the set of non-training examples is $1 - n^{c - \frac{1}{3}\log\log n}$. Therefore, for $n$ sufficiently large, \begin{align*}E_{f^*, \sample} \left[\Pr_{S,S' \sim \mathcal{D}} \left[S, S' \not\in \sample\text{, }f^*(S) = O(\log n), \text{ and } f^*(S') = \Omega(\sqrt{n})\right]\right] &= \left(\frac{1- n^{c - \frac{1}{3}\log\log n}}{2}\right)^2\\ &\geq \frac{49}{512}.\end{align*}

The remainder of the proof follows as in the proof of Theorem~\ref{thm:sub_lower}.
\end{proof}

\section{Proofs from Section \ref{sec:additiveError}} \label{app:additiveError}

\begin{proof}[Proof of Lemma~\ref{lem:add_concentration}]

Suppose $S_1,S_2 \subseteq [n]$ are such that $f(S_1) + 2\gamma\left(1+\sqrt{\frac{2}{\xi}}\right) < f(S_2)$. We know from Lemma~\ref{lemma:add_tail} and the union bound that with probability at least $1-\xi$, \[f(S_i) - \gamma\left(1+\sqrt{\frac{2}{\xi}}\right) \leq p_0(S_i) \leq f(S_i) + \gamma\left(1+\sqrt{\frac{2}{\xi}}\right)\] for both $i \in \{1,2\}$. Therefore, with probability at least $1-\xi$, \[\vec{p}\cdot v(S_1) = p_0(S_1) \leq f(S_1) + \gamma\left(1+\sqrt{\frac{2}{\xi}}\right) \leq f(S_2) - \gamma\left(1+\sqrt{\frac{2}{\xi}}\right) \leq p_0(S_2) = \vec{p}\cdot v(S_2).\]

\end{proof}

\begin{proof}[Proof of Claim~\ref{claim:lin_sep_error_add}]
First, in order to derive the sample complexity result, we need to bound the VC dimension of each class of loss functions $\mathcal{L}_{ij} = \{L_{(\vec{p}_{ij},\theta_{ij})} \ | \ \vec{p}_{ij} \in \R^{n^k}, \theta_{ij} \in \R\},$ where \[L_{(\vec{p}_{ij},\theta_{ij})}(S) = \begin{cases} 1 &\text{if } f(S) \leq f(S_i) \mbox{ and } \vec{p}_{ij} \cdot v(S) > \theta_{ij}\\
&\text{or } f(S) \geq f(S_j) \mbox{ and } \vec{p}_{ij} \cdot v(S) < \theta_{ij}\\
0 &\text{otherwise} \end{cases}.\] By the same reasoning as in the proof of Claim~\ref{claim:lin_sep_error} , the VC dimension of each $\mathcal{L}_{ij}$ is simply $n^k.$ Therefore, by standard VC-dimension bounds, we need \begin{align*}m' &= O\left( \frac{m^2}{\epsilon^2}\left[n^{O\left(\frac{1}{\gamma^{4/5}}\log \frac{1}{\gamma}\right)} \ln \frac{m^2}{\epsilon} + \ln \frac{1}{\delta}\right]\right)\\ &= O\left( \frac{1}{\epsilon^4}\log\frac{1}{\epsilon\delta}\left[n^{O\left(\frac{1}{\gamma^{4/5}}\log \frac{1}{\gamma}\right)} \log \left(\frac{1}{\epsilon^3}\log\frac{1}{\epsilon\delta}\right) + \log \frac{1}{\delta}\right]\right)\end{align*} examples to ensure that \[P\left[|L_{ERM_{ij}(\sample_2)}(\sample_2) - L_{ERM_{ij}(\sample_2)}(\mathcal{D})| > \frac{\epsilon}{4m^2}\right] < \delta.\] Since we only add $(i,j)$ to $R$ if $L_{ERM_{ij}(\sample_2)}(\sample_2) < \frac{\epsilon}{4m^2}$, this means that with probability at least $1-\delta$, \[L_{ERM_{ij}(\sample_2)}(\mathcal{D}) = \Pr_{S \sim \D}
\left[\begin{array}{r} f(S) \leq f(S_i)
\mbox{ yet } \vec{p}_{ij} \cdot \chi(S) > \theta_{ij}\\
\text{or }f(S) \geq f(S_j)
\mbox{ yet } \vec{p}_{ij} \cdot \chi(S) < \theta_{ij} \end{array}\right] < \frac{\epsilon}{2m^2}.\] By a union bound over all $m^2$ pairs in $R$, we have that with probability at least $1-\delta$, \[\Pr_{S \sim \D}
\left[\begin{array}{r}\exists (i,j)\in R \text{ such that } f(S) \leq f(S_i)
\mbox{ yet } \vec{p}_{ij} \cdot \chi(S) > \theta_{ij}\\
\text{or }f(S) \geq f(S_j)
\mbox{ yet } \vec{p}_{ij} \cdot \chi(S) < \theta_{ij} \end{array}\right] < \frac{\epsilon}{2}.\]
\end{proof}

\begin{proof}[Proof of Theorem~\ref{thm:general_add}]
We combine Claim~\ref{claim:buckets_additive} and Claim~\ref{claim:lin_sep_error_add} to prove Theorem~\ref{thm:general_add}. To this end, let $g$ be the comparison function returned by Algorithm~\ref{alg:additive_error}. We want to bound the probability that for $S,S' \sim \mathcal{D}$, $f(S) + \beta < f(S')$ but $g$ predicts that $f(S') \leq f(S).$ Equivalently, we want to bound the probability that for $S,S' \sim \mathcal{D}$, $f(S) + \beta < f(S')$ but there exists $(i,j) \in R$ such that $\vec{p}_{ij} \cdot v(S) > \theta_{ij}$ and
$\vec{p}_{ij} \cdot v(S') < \theta_{ij}$.  To analyze this probability, we partition the pairs $(i,j)\in R$ into two sets:
\begin{align*}R_1 = \{(i,j) \ | \ f(S)\leq f(S_i) \text{ or } f(S') \geq f(S_j)\} \text{ and}\\
R_2 = \{(i,j) \ | \ f(S)> f(S_i) \text{ and } f(S') < f(S_j)\}.\end{align*}

Clearly, \begin{align*}\Pr_{S,S' \sim \mathcal{D}} [ f(S) + \beta < f(S') \text{ and } \exists (i,j) \in R \text{ such that } \vec{p}_{ij} \cdot v(S) > \theta_{ij} \text{ and } \vec{p}_{ij} \cdot v(S') < \theta_{ij}]\\
\leq \Pr_{S,S' \sim \mathcal{D}} [ f(S) + \beta < f(S') \text{ and } \exists (i,j) \in R_1 \text{ such that } \vec{p}_{ij} \cdot v(S) > \theta_{ij} \text{ and } \vec{p}_{ij} \cdot v(S') < \theta_{ij}]\\
+ \Pr_{S,S' \sim \mathcal{D}} [ f(S) + \beta < f(S') \text{ and } \exists (i,j) \in R_2 \text{ such that } \vec{p}_{ij} \cdot v(S) > \theta_{ij} \text{ and } \vec{p}_{ij} \cdot v(S') < \theta_{ij}].\end{align*}

First, notice that \begin{align*} &\Pr_{S,S' \sim \mathcal{D}} [ f(S) + \beta < f(S') \text{ and } \exists (i,j) \in R_1 \text{ such that } \vec{p}_{ij} \cdot v(S) > \theta_{ij} \text{ and } \vec{p}_{ij} \cdot v(S') < \theta_{ij}]\\
\leq &\Pr_{S,S' \sim \mathcal{D}} \left[ \begin{array}{r}f(S) + \beta < f(S') \text{ and } \exists (i,j) \in R\text{ such that } f(S) \leq f(S_i)
\mbox{ yet } \vec{p}_{ij} \cdot \chi(S) > \theta_{ij}]\\
\text{or }f(S') \geq f(S_j)
\mbox{ yet } \vec{p}_{ij} \cdot \chi(S') < \theta_{ij}\end{array}\right].\end{align*}

From Claim~\ref{claim:lin_sep_error_add}, with probability at least $1-\frac{\delta}{2}$, this probability is at most $\frac{\epsilon}{2}$.

Next, we analyze \[\Pr_{S,S' \sim \mathcal{D}} [ f(S) + \beta < f(S') \text{ and } \exists (i,j) \in R_2 \text{ such that } \vec{p}_{ij} \cdot v(S) > \theta_{ij} \text{ and } \vec{p}_{ij} \cdot v(S') < \theta_{ij}].\] Recall that the algorithm
removed all non-minimal pairs from $R$. Therefore, the probability that there exists $(i,j) \in R_2$ is simply the probability that there does not exist $S_i,S_j \in \sample_1$ such that $f(S) + \beta \leq  f(S_i) + \beta < f(S_j) \leq f(S')$ and $(i,j) \in R$. Therefore,
\begin{align*} &\Pr_{S,S' \sim \mathcal{D}} [ f(S) + \beta < f(S') \text{ and } \exists (i,j) \in R_2 \text{ such that } \vec{p}_{ij} \cdot v(S) > \theta_{ij} \text{ and } \vec{p}_{ij} \cdot v(S') < \theta_{ij}]\\
\leq &\Pr_{S,S' \sim \D}\left[ f(S) + \beta < f(S') \mbox{ and } \not\exists (i,j) \in R \text{ such that }
 f(S) + \beta \leq f(S_i) + \beta < f(S_j) \leq f(S')\right].\end{align*} By using confidence and accuracy parameters $\delta/2$ and $\epsilon/2$, respectively, in Claim~\ref{claim:buckets_additive}, we have that with probability at least $1-\frac{\delta}{2}$, this probability is at most $\frac{\epsilon}{2}$.

Putting these bounds together, we have that with probability least $1-\delta,$ \[\Pr_{S,S' \sim \mathcal{D}} [ f(S) + \beta < f(S') \text{ and } \exists (i,j) \in R \text{ such that } \vec{p}_{ij} \cdot v(S) > \theta_{ij} \text{ and } \vec{p}_{ij} \cdot v(S') < \theta_{ij}] < \epsilon.\] Therefore, if $g$ is the classifier that Algorithm~\ref{alg:additive_error} outputs, with probability at least $1-\delta$, the probability that $f(S) + \beta < f(S')$ and $g$ predicts incorrectly is at most $\epsilon$.
Therefore, we have the desired result.
\end{proof}
\section{Proofs from Section~\ref{sec:otherCombinatorial}}\label{app:other_combinatorial}

\begin{proof}[Proof of Theorem~\ref{thm:general_valuation}] The interaction function allows us to express $f \in F_k$ as a linear function in $n^k$-dimensional space. In particular, let $S_1, \dots, S_{n^k}$ be an ordering of all subsets of $[n]$ of size at most $k$ and define $\chi_k(S)$ to be a vector in $\R^{n^k}$ whose $i^{th}$ component is 1 if $S_i \cap S \not= \emptyset$ and 0 otherwise. Next, let $\vec{g}_k$ be a vector in $\R^{n^k}$ whose $i^{th}$ component is $g(S_i)$. Then $f(S) = \vec{g}_k \cdot \chi_k(S).$

This suggests a straightforward adjustment to Algorithm~\ref{alg:active_learning}: If we know that the underlying valuation function $f$ is in $F_k$, then we can map each sample $S \subseteq [n]$ to $\chi_k(S)$ and attempt to learn linear threshold functions $w_{ij}, \theta_{ij}$ over $\R^{n^k}$ rather than $\R^n$ for all $i,j \in \sample_1$.

Note the sample complexity dependence on $\epsilon$ drops from $\frac{1}{\epsilon^3}$ (in Theorem \ref{thm:general}) to $\frac{1}{\epsilon^2}$ since we only need to
attempt to learn $w_{ij},\theta{ij}$ for adjacent $S_i,S_j\in \sample_1$ in the ordered list, because it is not required $f(S_i)$
and $f(S_j)$ need to be sufficiently far apart to guarantee successfully learning $w_{ij}$ and $\theta_{ij}$. 
This implies the union bound is over $m$ events instead of $m^2$ events.
\end{proof}

\begin{proof}[Proof of Theorem~\ref{thm:Fourier_sparse}]
We know that for any $S \subseteq [n]$, $f(S) = \sum_{T \subseteq [n]} \hat{f}(T) \chi_{T}(S),$ where $\chi_T(S) = (-1)^{|T \cap S|}.$ Since the Fourier support of $f$ is contained in $\mathcal{P}$, this equation simplifies to $f(S) = \sum_{T \subseteq \mathcal{P}} \hat{f}(T) \chi_{T}(S).$ Let $T_1, \dots, T_{|\mathcal{P}|}$ be an ordering of $\mathcal{P}$, and let $\vec{w} \in \R^{\mathcal{P}}$ be defined such that $w[i] = \hat{f}(T_i)$. By assumption, $\vec{w}$ has $k$ non-zero entries. If we map $S$ to $\R^{\mathcal{P}}$ by defining $v(S)$ to be a vector such that the $i^{th}$ component is $\chi_{T_i}(S)$, then $f(S) = \vec{w} \cdot v(S)$.

Therefore, in Algorithm~\ref{alg:active_learning}, we may attempt to learn linear threshold functions $w_{ij}, \theta_{ij}$ over $\R^{|\mathcal{P}|}$ rather than $\R^n$ for all $i,j \in \sample_1$. Since the VC dimension of all $k$-sparse halfspaces in $\R^{|\mathcal{P}|}$ is $k \log |\mathcal{P}|$, we can learn these linear separators to the precision required in the proof of Algorithm~\ref{alg:active_learning} by using $\tilde O\left(\frac{k}{\epsilon^3}\right)$ examples in time polynomial in $|\mathcal{P}|^{\Theta (k)}$, $1/\epsilon$, and $1/\delta$ \cite{feldman2014open}.

For the same reason as in Theorem \ref{thm:general_valuation}, the dependence on $\epsilon$ in the sample complexity can be lowered to $\frac{1}{\epsilon^2}$.
\end{proof}

\begin{remark} By ignoring the sparsity of the linear separators, we can learn them using $\tilde O\left(\frac{|\mathcal{P}|}{\epsilon^3}\right)$ examples in time polynomial in $|\mathcal{P}|$, $1/\epsilon$, and $1/\delta$. 
\end{remark}

\begin{proof}[Proof of Theorem~\ref{thm:coverage}]
To begin with, we rely on the following result from \cite{Sketching}.

\begin{theorem}{\cite{Sketching}}\label{thm:sketching}
For any coverage function $c: 2^{[n]} \to \R_{\geq 0}$, there exists a coverage function $\hat{c}$ on a universe $U'$ with $|U'| \leq \frac{27n^2}{\epsilon^2}$ such that for all $S \in 2^{[n]}$, $c(S)/(1+\epsilon) \leq \hat{c}(S) \leq c(S)$ with probability at least $1-2^{n+1}e^{-n}$.
\end{theorem}

We also use the following lemma from \cite{feldman2014learning}.

\begin{lemma}{\cite{feldman2014learning}}\label{lem:coverage_linear}
A function $c: 2^{[n]} \to \R_{\geq 0}$ is a coverage function on some universe $U$ if and only if there exist non-negative coefficients $\alpha_S$ for every $S \subseteq [n]$, $S \not= \emptyset$ such that $c(T) = \sum_{S \subseteq [n], S\not= \emptyset} \alpha_S \cdot \textsf{\emph{OR}}_S(T)$, and at most $|U|$ of the coefficients $\alpha_S$ are non-zero.
\end{lemma}

Here, $\textsf{OR}_S: 2^{[n]} \to \{0,1\}$ is defined such that for any $T \subseteq [n]$, $\textsf{OR}_S(T) = 0$ if and only if $T \subseteq S$.

Now, let $S_1, \dots, S_{2^n-1}$ be an ordering of $2^{[n]} \setminus \emptyset$ and for $S \subseteq [n]$, let $v(S) \in \{0,1\}^{2^n-1}$ be the vector defined as $v(S)[i] = \textsf{OR}_{S_i}(S)$. We know from Theorem~\ref{thm:sketching} and Lemma~\ref{lem:coverage_linear} that with probability at least $1-2^{n+1}e^{-n}$, there exists a vector $\vec{\alpha} \in \R^{2^n-1}$ such that for all $S \subseteq [n]$ and $\epsilon \in (0,1)$, $c(S)/(1+\epsilon) \leq \vec{\alpha} \cdot v(S) \leq c(S).$ Moreover, from Lemma~\ref{lem:coverage_linear}, we know that $\vec{\alpha}$ has at most $\frac{27 n^2}{\epsilon^2}$ non-zero entries. Therefore, with probability at least $1-2^{n+1}e^{-n}$, for any $S,S' \subseteq [n]$, if $(1+\epsilon)c(S) \leq c(S')$, then $v(S) \cdot \vec{\alpha} \leq c(S) \leq c(S')/(1+\epsilon) \leq v(S') \cdot \vec{\alpha}$. This means that in Algorithm 1, for each $S_i,S_j \in \sample_1$ such that $c(S_i) \leq c(S_j)$, we can solve for a linear separator $\alpha_{ij}, \theta_{ij}$ such that $v(S) \cdot \alpha_{ij} < \theta_{ij}$ if $c(S) \leq c(S_i)$ and $v(S) \cdot \alpha_{ij} > \theta_{ij}$ if $c(S) \geq c(S_j)$. With probability at least $1-2^{n+1}e^{-n}$, such a linear threshold function will exists for all $S_i,S_j$ such that $(1+\epsilon) c(S_i)\leq c(S_j)$.

It is well known that the VC dimension of the class of linear threshold functions $(\vec{w}, w_0)$ over $\R^d$ such that $||\vec{w}||_0 \leq r$ has VC dimension $O(r\log d)$ (ex. \cite{neylon2006sparse}). Therefore, the class of linear threshold functions Algorithm 1 learns over has VC dimension $O(|U'| \log (2^n - 1)) = O(n^3/\epsilon^2)$. Therefore, we change the size of $\sample_2$ to be \[|\sample_2| = O\left(\frac{m^2}{\epsilon}\left[\frac{n^3}{\epsilon^2} \log \frac{1}{\epsilon} + \log\frac{1}{\delta m^2}\right]\right),\] where $m = \frac{1}{\epsilon} \log \frac{1}{\epsilon \delta}$. Moreover, we need to map each $S \in \sample_2$ to $\R^{2^n-1}$ in order to learn the linear separators $\alpha_{ij}, \theta_{ij}$. To do this, let $S_1, \dots, S_{2^n-1}$ be an ordering of $2^{[n]} \setminus \emptyset$. Then we define the mapping $v: 2^{[n]} \to \R^{2^n-1}$ such that for all $S \subseteq [n]$, the $i^{th}$ component of $v(S)$ is $\textsf{OR}_{S_i}(S)$. With these changes, the analysis of Algorithm 1 in Section~\ref{sec:multiplicativeError} holds, with $\alpha(n) = 1+\epsilon$.

\end{proof}

\end{document}